\documentclass[sigconf]{acmart}

\pdfoutput=1

\usepackage{CJKutf8}
\usepackage{array}
\usepackage[utf8]{inputenc} 
\usepackage[T1]{fontenc}    
\usepackage[table]{xcolor}  
\usepackage{hyperref}       
\usepackage{url}            
\usepackage{booktabs}       
\usepackage{tabularx}
\usepackage{amsfonts}       
\usepackage{nicefrac}       
\usepackage{microtype}      
\usepackage{graphicx}
\usepackage{amsmath}
\usepackage{wrapfig}
\usepackage{multirow}
\usepackage{comment}
\usepackage{caption}

\usepackage{subcaption}
\usepackage{pifont}

\newcommand{\ourmethod}{\textit{PAST}}

\definecolor{myyellow}{RGB}{190,144,0}
\definecolor{mygreen}{RGB}{0,136,51}
\definecolor{myblue}{RGB}{0,102,204}

\begin{document}

\begin{CJK*}{UTF8}{gbsn}


\title{\ourmethod: Prompt-Adaptive Sampling Termination for Efficient Diffusion Model}


\author{Renye Yan}
\authornote{Renye Yan and Jikang Cheng contributed equally to this work.}
\affiliation{%
  \institution{Peking University}
  \city{Beijing}
  \country{China}
}

\author{Jikang Cheng}
\authornotemark[1]
\affiliation{%
  \institution{Peking University}
  \city{Beijing}
  \country{China}
}

\author{You Wu}
\affiliation{%
  \institution{Nanjing University}
  \city{Nanjing}
  \country{China}
}

\author{Wei Peng}
\affiliation{%
  \institution{Stanford University}
  \city{Stanford}
  \state{California}
  \country{USA}
}

\author{Zongwei Wang}
\affiliation{%
  \institution{Peking University}
  \city{Beijing}
  \country{China}
}


\author{Ling Liang}
\authornote{Ling Liang and Yimao Cai are the corresponding authors.}
\affiliation{%
  \institution{Peking University}
  \city{Beijing}
  \country{China}
}

\author{Yimao Cai}
\authornotemark[2]
\affiliation{%
  \institution{Peking University}
  \city{Beijing}
  \country{China}
}


\renewcommand{\shortauthors}{Yan et al.}

\begin{abstract}

While diffusion models have made significant progress in text-to-image tasks, they still exhibit limitations when directly optimizing downstream objectives. Although Reinforcement Learning (RL) enables targeted optimization, existing methods are generally constrained by low-efficiency fine-tuning and sparse rewards. To address these challenges, we propose \ourmethod{}, which provides differentiated rewards while adaptively regulating training episode length by jointly perceiving denoising progress and prompt difficulty. Specifically, we design an intrinsic reward paradigm to compensate for sparse extrinsic rewards and guide the model to explore paths that diverge more efficiently from noise patterns. We further provide theoretical justification for intrinsic rewards. Then, \ourmethod{} dynamically monitors denoising completion and semantic alignment between image structures and prompt semantics. When both metrics satisfy generation requirements, the system adaptively terminates training. This enables appropriate allocation of episode lengths based on prompt difficulty and the current generation process. Finally, based on the predicted residual noise level, we establish a dual adaptive coordination mechanism. Specifically, it not only balances the extrinsic and intrinsic rewards but also balances the exploration and convergence. Experimental results demonstrate that \ourmethod{} enhances computational efficiency of existing RL fine-tuning methods by up to 66.7\%, while improving preference optimization quality by up to 29.5\% through its dual adaptive regulation mechanism.

\end{abstract}







\maketitle

\section{Introduction}
\label{sec:intro}

In recent years, diffusion models~\cite{ho2020denoising,ramesh2022hierarchical,rombach2022ldm,wei2026team,li2026videococo,yan2025entropy,yan2026pixel} have demonstrated powerful capabilities in text-to-image (T2I)~\cite{ramesh2022hierarchical,betker2023improving,clark2023directly,lin2024evaluating,esser2024scaling,lee2024direct,li2024aligning,yuan2024self,ding2025rass} generation. These models generate samples that match the target data distribution by gradually removing noise. They have been widely applied in tasks such as image synthesis~\cite{saharia2022photorealistic,nichol2021glide,zhang2025trustclip}, video generation~\cite{ho2022imagen,ho2022video,zhang2026rep}, 3D shape modeling~\cite{zhou20213d,xu2023dream3d}, etc. However, diffusion models typically aim to match the training data distribution, lacking direct optimization for downstream tasks. 
Against this backdrop, reinforcement learning (RL)~\cite{sutton2018reinforcement,gan2024reflective,gan2024transductive} fine-tuning has emerged as a promising solution for adapting diffusion models to downstream objectives~\cite{black2023training,fan2024reinforcement}. This suitability stems from the sequential structure of diffusion models, where each timestep constitutes a decision about how the sample is denoised. Building on this, mainstream RL-based approaches model the diffusion process as a Markov Decision Process and leverage reward-based~\cite{zhong2025comprehensive} or preference-based optimization strategies~\cite{black2023training,fan2024reinforcement,yang2024using} to better align the model with downstream objectives.


While RL fine-tuning shows promise for enhancing diffusion model performance, it still faces two critical challenges. (1) Sparsity and delay of reward signals: Task rewards are only provided upon completion of the generation process, resulting in a lack of timely feedback throughout the generation procedure. This makes it challenging to correct denoising deviations in time, often exacerbating structural distortions in the generated output and increasing the risk of reward hacking ~\citep{franceschelli2024reinforcement,black2023training,franceschelli2024reinforcement,zhang2024confronting}. (2) High computational cost: Although existing RL fine-tuning frameworks significantly improve the model's ability to fit downstream tasks, such performance gains inevitably come with substantial computational overhead. This leads to a notable increase in overall training costs and extended training durations, which limit their scalability and practical efficiency in real-world applications ~\citep{zhang2025generalization,fan2023dpok,yang2024using,kim2025test,xie2025dymo,zhang2025generalization}.

\begin{figure*}[htbp]
    \centering
    \includegraphics[width=0.99\linewidth]{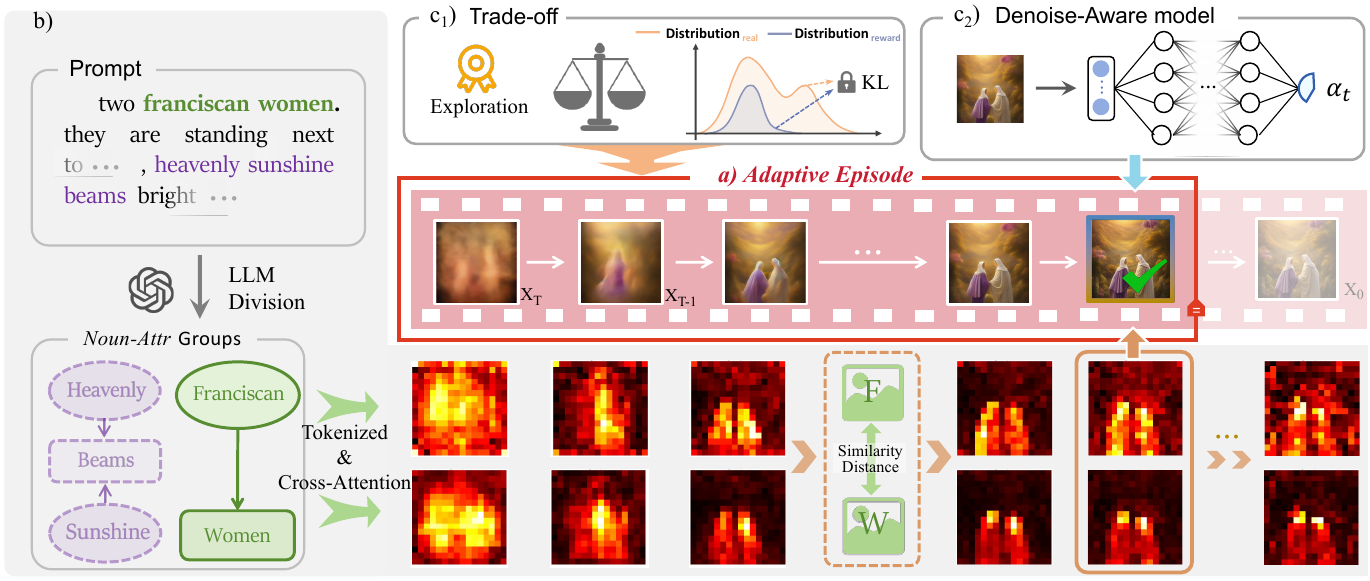}
    \caption{\textbf{Overview of \ourmethod{}}. Tackling the challenges of high computational cost and sparse rewards of existing RL fine-tuning methods, we introduce a) adaptive episode control based on prompt complexity and denoising progress, enabling a broader search range for high reward policy preserving real distribution. b) We design an intrinsic motivation based on fine-grained text-image alignment to alleviate reward sparsity. c) Both the intrinsic-extrinsic reward trade-off and the exploration-convergence trade-off are dynamically coordinated by a denoise evaluator.}
    \label{fig:placeholder}
\end{figure*}


To address reward hacking~\cite{yan2026less}, recent studies have introduced human preference modeling~\cite{uehara2024understanding,rafailov2023direct,yang2024using,wallace2024diffusion} into T2I diffusion tasks, drawing inspiration from Reinforcement Learning from Human Feedback (RLHF) in large language models (LLMs)~\cite{li2024leveraging,wang2025rare,li2025preference,ma2025inner}. The key point is to train a human preference score model that ranks generated image pairs, replacing static reward functions based on prior knowledge~\cite{kirstain2023pick,hpsv2}. This approach bypasses the reliance on fixed rewards and helps mitigate reward hacking during training. However, these methods still only evaluate the final images, and the computation cost issue remains, failing to alleviate the fundamental issue of sparse optimization signals.


To reduce the high computational cost of RL fine-tuning for diffusion models, recent methods attempt inference-time guidance without additional training~\cite{yu2023freedom,jin2023training,xie2025dymo}. These approaches use prior knowledge to guide generation without modifying model parameters. However, due to the lack of task-specific adaptation, their fitting capacity is limited, often trading off generation quality for efficiency. This creates a dilemma between Training-free: limited alignment gains and With Training: noticeable alignment gains but higher cost~\cite{kim2025test}. Thus, how to substantially reduce the cost of T2I generation while ensuring that output quality does not degrade has become a critical problem that urgently needs to be addressed. 
\textit{A detailed review of \textbf{Related Works} is provided in Appendix~\ref{appendix:related_work}}.


Motivated by the need to jointly address sparse rewards, reward hacking, and excessive computational cost, we propose Prompt-Adaptive Sampling Termination (\ourmethod{}). A lightweight RL fine-tuning plugin that integrates seamlessly into existing diffusion frameworks. \ourmethod{} features two core components: first, an intrinsic reward based on denoising progress to mitigate sparse rewards and reward hacking. Furthermore, we offer a theoretical justification for the inclusion of the intrinsic reward; second, a prompt-adaptive episode termination strategy that dynamically adjusts RL training steps to reduce computation cost. Specifically, the intrinsic reward guides the policy toward faster denoising while balancing exploration and convergence. Building on this, \ourmethod{} employs an adaptive episode strategy: a lightweight denoising predictor estimates the image’s cleanliness, while cross-attention responses measure prompt-image alignment. Once both indicators stabilize, the generation process is considered complete, and sampling is terminated early with the current state returned as the output.


Unlike prior RL methods that use fixed-length episodes for all prompts, \ourmethod{} adaptively adjusts episode length based on prompt complexity and generation progress. Thus, \ourmethod{} can 
skip the redundant parts whose marginal returns have already saturated during the policy training phase, thereby bringing three benefits: first, it reduces the impact of low-quality samples, thereby improving training stability; second, it mitigates reward hacking from over-optimization; and third, it significantly cuts computational cost.


We conduct comprehensive evaluations of \ourmethod{} across standard downstream metrics, advanced prompt datasets, ablation studies, and cross-task generalization tests. \ourmethod{} consistently outperforms mainstream baselines. It achieves comparable or improved performance while significantly reducing computational cost. These results demonstrate both its optimization effectiveness and broad adaptability. Overall, the contributions of this paper are as follows:


\begin{itemize}

    \item We break the fixed-episode paradigm in existing RL fine-tuning by introducing the adaptive episode control based on prompt complexity and denoising progress, reducing training computational cost.

    \item We introduce an intrinsic motivation mechanism for the diffusion model to alleviate reward sparsity. In addition, we justify the introduction of the intrinsic reward.
    
    \item By trade-off intrinsic-extrinsic rewards and exploration-\\convergence, we achieve a more effective reward search under a real distribution.

\end{itemize}

\begin{figure*}[htbp]
    \centering
    \includegraphics[width=1\linewidth]{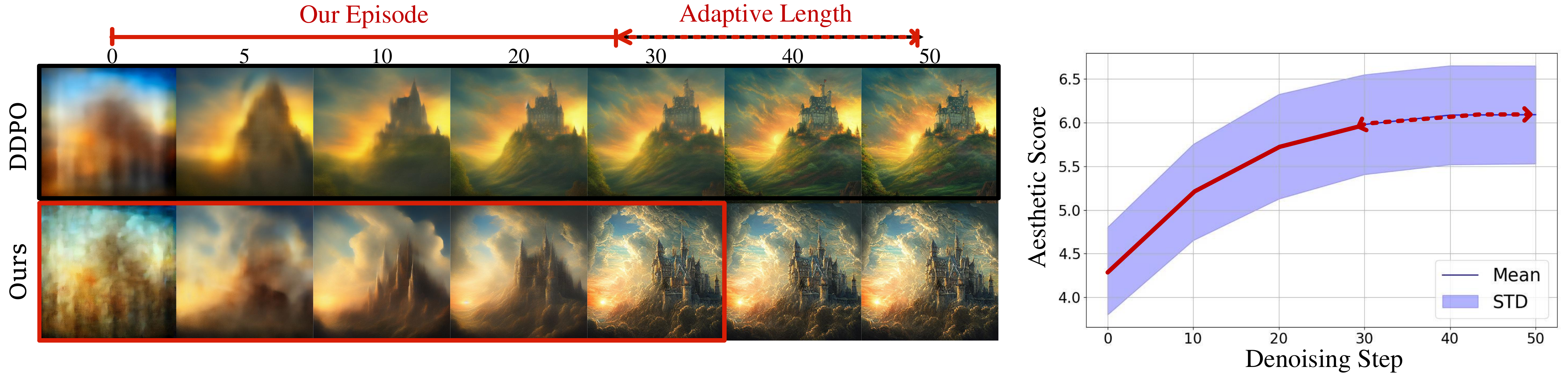}
    \caption{Image quality during denoising. The quality improvement becomes marginal in the last segment of denoising steps~\cite{ma2024deepcache}. Compared to traditional methods that go through the entire denoising process, our method adapts episode lengths to focus the training on the steps where image quality is significantly improving to enhance efficiency.}
    \label{fig:intro}
\end{figure*}

\section{Method}
\label{sec:method}

In this seciton we propose \ourmethod{} that improves sampling efficiency and stabilizes policy learning through intrinsic motivation and adaptive epsoide termination. Without altering diffusion framework or training paradigms, \ourmethod{} can be integrated into existing RL methods.
In Sec.~\ref{intrinsic}, we design intrinsic rewards to complement sparse extrinsic rewards and accelerate noise removal. We address the extrinsic-intrinsic rewards and exploration-convergence trade-off dilemma in Sec.~\ref{sec:trade-off}. Sec.~\ref{sec:ada} presents an adaptive episode termination mechanism that dynamically discards low-value steps in training to enhance sample efficiency.

\subsection{RL for Diffusion Model}\label{sec:21}

\subsubsection{MDP in Diffusion Model.} 
We formulate the denoising process as a Markov Decision Process (MDP), where $p_\theta(\mathbf{x}_{0:T}\mid\mathbf{c})$ is defined as the diffusion model for text-to-image generation. The prompt $\mathbf{c}$ follows the distribution $p(\mathbf{c})$. The MDP process involve
\begin{equation}
\label{denotes}
    \begin{aligned}
        \mathbf{s}_t &\triangleq \left(\mathbf{c}, \mathbf{x}_{T-t}\right), \\ 
        \mathbf{a}_t &\triangleq \mathbf{x}_{T-t-1}, \\ 
        \pi_\theta\left(\mathbf{a}_t \mid \mathbf{s}_t\right) &\triangleq p_\theta\left(\mathbf{x}_{T-t-1} \mid \mathbf{x}_{T-t}, \mathbf{c}\right), \\
        P_0\left(\mathbf{s}_0\right) &\triangleq \left(p(\mathbf{c}), \mathcal{N}(\mathbf{0}, \mathbf{I})\right),\\
        P\left(\mathbf{s}_{t+1} \mid \mathbf{s}_t, \mathbf{a}_t\right) &\triangleq \delta_{(\mathbf{c}, \mathbf{x_{T-t-1}})}. \\ 
    \end{aligned}
\end{equation}


Here,\( \mathbf{s}_t \) denotes all possible image states that can arise during the fine-tuning process, with each \( \mathbf{s}_t \) corresponding to an intermediate variable \( \mathbf{x}_{T-t} \) in the denoising trajectory. The action \( \mathbf{a}_t \) refers to the denoising output produced by the model at timestep \( t \), namely \( \mathbf{a}_t = \mathbf{x}_{T-t-1} \). This action is determined jointly by the noise prediction and the sampling noise. The remaining notation is defined as follows: $\pi_\theta$ represents the parameterized policy; $P_0$ denotes the initial state distribution; $\delta_{(\mathbf{c}, \mathbf{x_{T-t-1}})}$ denotes the Dirac distribution centered at $(\mathbf{c}, \mathbf{x_{T-t-1}})$; and \( P(\mathbf{s}_{t+1} \mid \mathbf{s}_t, \mathbf{a}_t) \) represents the state transition distribution, which specifies the conditional probability of the next state \( \mathbf{s}_{t+1} \) given the current state \( \mathbf{s}_t \) and action \( \mathbf{a}_t \).

Existing RL studies define the reward or feedback as:
\begin{equation}
\label{fake_reward}
        R\left(\mathbf{s}_t, \mathbf{a}_t\right) \triangleq 
        \begin{cases}
            r(\mathbf{s}_{t+1}) = r\left(\mathbf{x}_0, \mathbf{c}\right) & \text{if } t=T-1, \\
            0 & \text{otherwise}.
        \end{cases}.
\end{equation}
As can be seen, \( R(s_t, a_t) \) has a non-zero value only at the final MDP step, corresponding to the generated image. This sparse reward setting results in a lack of intermediate feedback signals during learning, making it difficult for the model to make dynamic adjustments and potentially leading to accumulated denoising errors~\cite{black2023training,franceschelli2024reinforcement}.

\subsubsection{Intrinsic Motivation}
\label{intrinsic}
With intrinsic motivation, the MDP introduces new symbols:
$R_{\text{ext}}\left(\mathbf{s}_t, \mathbf{a}_t\right)$ denotes the extrinsic rewards associated with specific downstream tasks, 
which is exactly equivalent to $R\left(\mathbf{s}_t, \mathbf{a}_t\right)$ in Eq.~\ref{fake_reward};
$R_{\text{int}}\left(\mathbf{s}_t, \mathbf{a}_t\right)$ denotes the intrinsic rewards;
By combining the intrinsic and extrinsic rewards with a fixed ratio $\lambda$, we obtain the total rewards 
$R_{\text{total}}\left(\mathbf{s}_t, \mathbf{a}_t\right)$, which is defined as:
\begin{equation*} 
\begin{aligned} 
R_{\text{total}}\left(\mathbf{s}_t, \mathbf{a}_t\right)  &= R_{\text{ext}}\left(\mathbf{s}_t, \mathbf{a}_t\right)  + \lambda R_{\text{int}}\left(\mathbf{s}_t, \mathbf{a}_t\right) .
\end{aligned} 
\end{equation*}
Where $R_{\text{int}}$ stands for a complement to the extrinsic reward, which providing fine-grained feedback at intermediate steps during the generation. This combination of intrinsic and extrinsic rewards effectively overcomes the limitations of sparse rewards.

The intrinsic reward mechanism in \ourmethod{} is designed to encourage exploration of faster denoising trajectories. Thus, the intrinsic reward is defined as the difference between each intermediate image $\mathbf{x}_t$ and the initial noise image $\mathbf{x}_T$ in the diffusion process, formally:
\begin{equation}
\label{chuint}
 R_{\text{int}}\left(\mathbf{s}_t, \mathbf{a}_t\right) = || \mathbf{x}_{t} - \mathbf{x}_T ||_2^2.
\end{equation}
In this paper, $R_{\text{int}}(\mathbf{s}_t, \mathbf{a}_t)$ is abbreviated as $R_{{int}}$. To guarantee reliable exploration, we impose a dynamic KL constraint to keep exploration consistent with the true distribution (see Sec.~\ref{sec:trade-off}).

\noindent \textbf{Role of Intrinsic Reward:}
\begin{itemize}
    \item Promoting fast denoising: Supporting experimental results can be found in Fig.~\ref{fig:std} \& Fig.~\ref{fig:tra_rebuttal} (in the appendix).
    \item Encouraging diverse generation during denoising: Supporting experimental results can be found in Tab.~\ref{tab:tv}, and Fig.~\ref{fig:epochD-re} \&~\ref{fig:div_all_past} \& ~\ref{fig:hacking} (in the appendix).
\end{itemize}

\noindent \textbf{Intrinsic Reward Rationality:}
\begin{itemize}
    \item Theoretical guarantees: the intrinsic reward is theoretically justified in Theorem.~\ref{thm1}.
    \item Adaptive KL: As described in Sec.~\ref{sec:trade-off}, we adaptively impose distributional constraints according to the training progress, ensuring that exploration driven by intrinsic rewards remains within the vicinity of the true data distribution. This adaptive mechanism enhances the robustness of the exploration process. Benefiting from the synergy between intrinsic rewards and adaptive constraints, the model is able to discover more valuable generative trajectories, thereby improving task reward acquisition while preventing degradation in fidelity (FID) during exploration (see Tab.~\ref{tab:tv}).
\end{itemize}


\subsubsection{Adaptive factor}
\label{adafactor}
We construct a Denoise-Aware Model (DAM) that uses the initial noise image \( \mathbf{x}_{T} \) and the final generated result \( \mathbf{x}_{0} \) as training data to assess in real time the degree of denoising cleanliness of the current image. During training, \( \mathbf{x}_T \) is labeled as 0 and the fully denoised image \( \mathbf{x}_0 \) is labeled as 1. To clearly distinguish the regions close to the initial noise and close to the final result, we set the neighborhood ranges around these two types of samples as a tunable hyperparameter \( \nu \), whose value is provided in the appendix. Based on this design, DAM can determine in real time how clean the current generated state \( \mathbf{x}_t \) is. DAM is trained using a binary cross-entropy loss:
\begin{equation}
    \mathcal{L}_{\text{DAM}} = -\sum_{t} \left( y_t \log(f(\mathbf{x}_t)) + (1 - y_t) \log(1 - f(\mathbf{x}_t)) \right),
\end{equation}
where \( y_t \in \{0,1\} \) indicates whether the sample lies in the clean-side neighborhood (\( y_t = 1 \)) or the noise-side neighborhood (\( y_t = 0 \)) of the denoising trajectory, and \( f(\mathbf{x}_t) \) denotes the predicted probability that the sample belongs to the clean-side region. After training, we define \( \alpha_t = f(\mathbf{x}_t) \) as a real-time measure of the denoising progress. Thus, $\alpha_t \in [0,1]$ and represents the cleanliness level of the intermediate state $\mathbf{x}_t$.

We define the adaptive factor as \( (1 - \alpha_{t}) \) and \( \alpha_{t} \). Specifically, in the early stage of the denoising process, the image contains a large amount of noise, so \( \alpha_{t} \) tends to 0. As the noise level decreases, \( \alpha_{t} \) gradually increases. We use the adaptive factor to perceive the denoising progress in real time and achieve a dual adaptive trade-off(see Sec.~\ref{reward_trade} and Sec.~\ref{Exploration-Convergence} for details).

\subsection{Dynamic Trade-off in Diffusion Fine-tuning}
\label{sec:trade-off}
Although RL enhances the generative capability of diffusion models, it also introduces exploration–exploitation challenges~\cite{franceschelli2024reinforcement,hao2024reinforcement,uehara2024understanding}. This work implements dual adaptive balancing between intrinsic (exploration) and extrinsic (exploitation) rewards, as well as between exploration and KL regularization, enabling the optimized policy to achieve higher downstream rewards while remaining aligned with the true data distribution.

\subsubsection{Regulation of Extrinsic-Intrinsic Rewards}
\label{reward_trade}

At rewards level, we employ the adaptive factor \( (1 - \alpha_{t}) \) proposed in Sec.~\ref{adafactor} to balance the intrinsic and extrinsic rewards over time. In the early generation stages, when extrinsic rewards do not obviously reflect downstream task quality, the optimization is primarily guided by intrinsic rewards to mitigate sparsity and encourage exploration. As the generation progresses and the states approaches the final image, the weight of intrinsic rewards is gradually reduced, allowing extrinsic rewards to dominate and guide the policy to converge toward high-reward sample distributions stably. Therefore, the overall reward function with the adaptive mechanism is:
\begin{equation*} 
\begin{aligned} 
R_{\text{total}}\left(\mathbf{s}_t, \mathbf{a}_t\right)  &= R_{\text{ext}}\left(\mathbf{s}_t, \mathbf{a}_t\right)  + (1 - \alpha_{t}) \cdot R_{\text{int}}\left(\mathbf{s}_t, \mathbf{a}_t\right)\\
&=r\left(\mathbf{x}_0, \mathbf{c}\right)  + (1 - \alpha_{t}) \cdot || \mathbf{x}_{t} - \mathbf{x}_T ||_2^2 .
\end{aligned} 
\end{equation*}
The adaptive reward mechanism causes the model to use intrinsic rewards during the early denoising stage to address sparse feedback. In the late stage, this mechanism will shift to extrinsic rewards to maximize performance on downstream tasks.


\subsubsection{Exploration-Convergence Scheduling}
\label{Exploration-Convergence}


At the level of the generative distribution, we use adaptive factors \( (1 - \alpha_{t}) \) and \( \alpha_{t} \) to adjust exploring new generation trajectories and converging to the true distribution. In the early stage, we encourage exploration and faster denoising trajectories, providing flexibility for RL to determine episode lengths adaptively. As the generation progresses, the weight of the KL term gradually increases, ensuring stable convergence of the fine-tuning distribution to the actual target. Therefore, the global optimization objective is:
\begin{equation}\label{lossada1}
    \begin{aligned}
        \mathcal{J} =&
\mathbb{E}_{p(\mathbf{c})}\mathbb{E}_{p_{\theta}(\tau|\mathbf{c})}[r\left(\mathbf{x}_0, \mathbf{c}\right))
        - \sum_{t=1}^{T}\alpha_t KL(p_{\theta}(\cdot))||p_{pre}(\cdot))]
\\&
        +\mathbb{E}_{p(\mathbf{c})}\mathbb{E}_{p_{\theta}(\tau|\mathbf{c})}\sum_{t=0}^{T} (1-\alpha_t)R_{\text{int}}\left(\mathbf{s}_t, \mathbf{a}_t\right)(\mathbf{x}_t)
,
    \end{aligned}
\end{equation}
where $p_{\theta}$ and $p_{pre}$ denote the fine-tuned and original model. The prompt $\mathbf{c}$ is drawn from the distribution $p(\mathbf{c})$. The $\tau= (\mathbf{x}_{T}, \mathbf{x}_{T-1}, \ldots,$ $ \mathbf{x}_{0})$ denotes the entire denoising trajectory.

\begin{theorem}\label{thm1}
    Given the objective of Eq. (\ref{lossada1}), the optimal policy $p_{\theta^*}$ has the following expression:
    \begin{equation}\label{aa}
    \begin{aligned}
        &p_{\theta^*}(\mathbf{x}_{t-1}|\mathbf{x}_t, \mathbf{c})\\
        \propto & \exp\{\!\frac{Q_{ext}^*(\mathbf{x}_{t\!-\!1},\! \mathbf{c}\!)\!+ Q_{int}^*(\mathbf{x}_{t\!-\!1},\! \mathbf{c}\!)\!+\!\alpha_{t-1}\! \log{p_{pre}(\mathbf{x}_{t\!-\!1}\!|\!\mathbf{x}_t,\!\mathbf{c}\!)}}{\alpha_{t-1}}\},
    \end{aligned}
\end{equation}
where $ Q_{ext}^*\!(\mathbf{x}_{t\!-\!1}, \mathbf{c})\! =\! r(\mathbf{x}_0,\mathbf{c}\!)I_{\{t\!-\!1\!=\!0\}}\!+ \alpha_{t-2}\cdot\log \sum_{\mathbf{x}_{t-2}}\exp\{\frac{Q_{ext}^*(\mathbf{x}_{t-2},\mathbf{c})}{\alpha_{t-2}}\}$ and $Q_{int}^*\!(\mathbf{x}_{t\!-\!1}, \mathbf{c})\! =\! (1-\alpha_{t-1})R_{int}(\mathbf{x}_{t-1}) + \alpha_{t-2}\cdot$ 
$\log \sum_{\mathbf{x}_{t-2}}\exp\{\frac{1}{\alpha_{t-2}}($ $Q_{int}^*(\mathbf{x}_{t-2},\mathbf{c})+\alpha_{t-2}\log{p_{pre}(\mathbf{x}_{t-2}|\mathbf{x}_{t-1},\mathbf{c})})\}$.
\end{theorem}
\noindent \textbf{Theorem’s conclusion.} The theorem guarantees the convergence of Eq.~\ref{lossada1}, facilitating a stable training process in the experiments. \underline{The proof is as follows:}


\begin{proof}
According to~\cite{DaiS0XHLCS18}, we have 
\begin{align*}
p_{\theta^*}(x_{t-1}|x_t,z) \varpropto \exp\{\frac{Q^*(x_{t-1}, z)+\alpha_{t-1}\log{p_{pre}(x_{t-1}|x_t,z)}}{\alpha_{t-1}}\},
\end{align*}
where
$ Q^*(x_{t-1}, z) = R(x_{t-1})+ 
\alpha_{t-2}
\log \sum_{x_{t-2}}\exp\{\frac{1}{\alpha_{t-2}}(Q^*(x_{t-2},z)+\alpha_{t-2} \log{p_{pre}(x_{t-2}|x_{t-1},z)})\} 
$.

We know that $ R(x_t) = R_{ext}(x_t) +(1-\alpha_{t}) R_{int}(x_t)$.
According to $Q^*(x_{t-1},z)$, 
defined $$ Q_{ext}^*\!(x_{t\!-\!1}, z)\! =\! r(x_0,\!z)I_{\{t\!-\!1\!=\!0\}}\!+ \alpha_{t-2}\cdot\log \sum_{x_{t-2}}\exp\{\frac{Q_{ext}^*(x_{t-2},z)}{\alpha_{t-2}}\}$$
and $ Q_{int}^*\!(x_{t\!-\!1}, z)\! =\! (1-\alpha_{t-1})R_{int}(x_{t-1})+ \alpha_{t-2}\cdot$ 
$\log \sum_{x_{t-2}}\exp\{\frac{1}{\alpha_{t-2}}($ $Q_{int}^*(x_{t-2},z)+\alpha_{t-2}\log{p_{pre}(x_{t-2}|x_{t-1},z)})\}$,
we have
\begin{align*}
		&\exp\{\frac{Q^*(x_{t-1},z)}{\alpha_{t-2}}\}\\
        =&\exp\{\frac{R(x_{t-1})}{\alpha_{t-2}}\}
		\sum_{x_{t-2}}\exp\{\frac{Q^*(x_{t-2},z)+\alpha_{t-2} \log{p_{pre}(x_{t-2}|x_{t-1},z)}}{\alpha_{t-2}}\} \\
		=& \exp\{\frac{R_{int}(x_{t-1})}{\alpha_{t-2}}\}
		\sum_{x_{t-2}}\exp\{\frac{R_{int}(x_{t-2})}{\alpha_{t-3}}\}
		\exp\{\frac{\log{p_{pre}(x_{t-2}|x_{t-1},z)}}{\alpha_{t-2}}\}\\
		&\sum_{x_{t-3}}\exp\{\frac{Q^*(x_{t-3},z)+\alpha_{t-3} \log{p_{pre}(x_{t-3}|x_{t-2},z)}}{\alpha_{t-3}}\} \\
		= & \exp\{\frac{Q_{ext}^*(x_{t-1}, z) +Q_{int}^*(x_{t-1}, z)}{\alpha_{t-1}}\}
	\end{align*}





Therefore, we have
\begin{equation*}
    Q^*(x_{t-1},z) = Q^*_{ext}(x_{t-1},z) +Q^*_{int}(x_{t-1},z) 
\end{equation*}

\end{proof}

\begin{figure*}[htbp]
    \centering
    \includegraphics[width=1\linewidth]{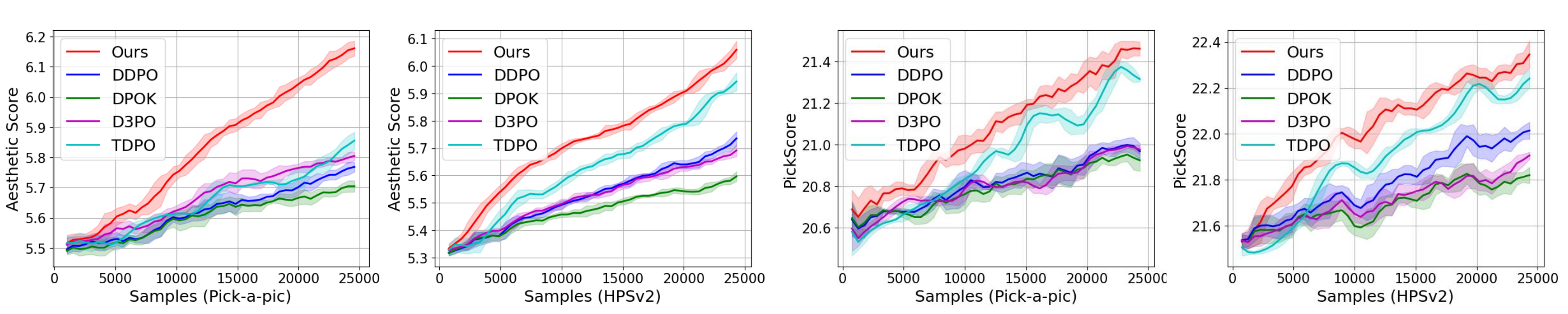}
    \caption{Reward optimization results with two datasets (\textit{i.e.}, HPSv2 and Pick-a-pic) and two reward functions (\textit{i.e.}, Aesthetic Score and PickScore). Our method performs the most efficiently in optimizing the rewards.}
    \label{fig:main}
    \vspace{-0.4cm}
\end{figure*}
\subsection{Adaptive Episode Length Control}
\label{sec:ada}
\subsubsection{Sementic alignment evaluator}
Evaluation functions like CLIP are not only time-consuming but also struggle to provide accurate real-time assessments of semantic alignment during denoising~\cite{dhariwal2021beat}. In conditional diffusion models, prompt embeddings are incorporated through cross-attention. Therefore, it is reasonable to assume that each token's attention map reflects its semantic content.
Based on this, we propose a metric to quantify prompt-image alignment during denoising by evaluating the similarity of attention maps within a group of associated samples. An associated sample group is defined as a \textit{Noun} from the prompt along with all its associated attributive (\textit{Attr}). Since the tokens in a group jointly describe the same visual region in the generated image, their corresponding attention maps should focus on similar regions. The similarity among these maps can thus serve as a metric for semantic alignment. Thanks to the rapid development of large language models (LLMs), we can rely on them to efficiently decompose prompts, thereby obtaining the \textit{Noun}-\textit{Attr} sample groups of each training prompt.

Formally, given a prompt $\mathbf{c}=\{\mathbf{G}_0,...,\mathbf{G}_m\}$ contains $m$ \textit{Noun}-\textit{Attr} group, we select one group 
$\mathbf{G}_i=\{\text{tok}_i^0,...\text{tok}_i^n\}$ and compute the pairwise cosine distance between their attention maps. The average of these distances is used as an indicator of alignment quality. 
\begin{equation}\label{lossada}
    \begin{aligned}
        &\text{Dist}(\mathbf{G}_i)=\\
        & \frac{1}{N(N-1)} \sum_{\substack{j=1 \\ j \ne k}}^{N} \sum_{k=1}^{N} \left(1 - \frac{\langle A_t(\text{tok}_i^j), A_t(\text{tok}_i^k) \rangle}{\|A_t(\text{tok}_i^j)\| \cdot \|A_t(\text{tok}_i^k)\|} \right),
    \end{aligned}
\end{equation}
where $A_t$ represents the token's attention map of $t^{th}$ denosing step, $N$ represents the total length of token. The alignment metric of $\mathbf{c}$ at the current step is: 
\begin{equation}
\alpha^s_t =  \frac{1}{m} \sum_{i=1}^{m} \text{Dist}(G_i).
\end{equation}
The alignment metric can then represent the semantic signal for episode length control. 
\\
\subsubsection{Denoise evalulator}
Note that the output of the Denoise-Aware Model  $\alpha_t$  could be used to evaluate the denoising status of the intermediate step. To align with $\alpha^s_t$, we re-denote $\alpha_t$ as $\alpha^d_t$ in this subsection, representing the denoising completion indicator.

\begin{figure}[htbp]
    \centering
    \includegraphics[width=1\linewidth]{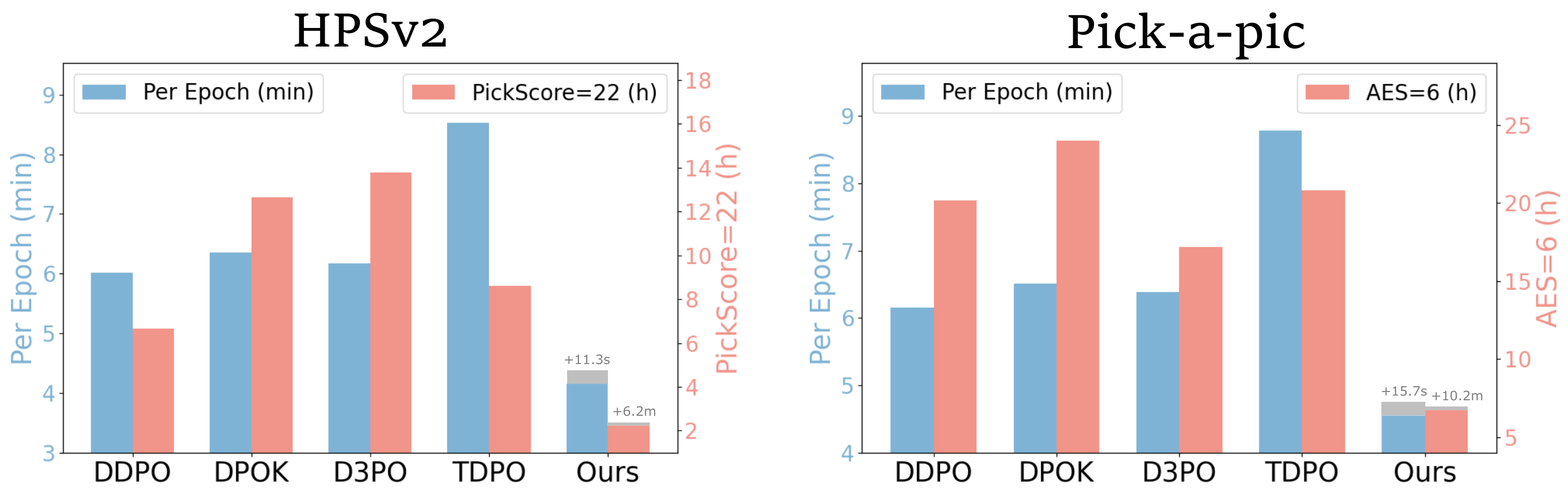}
    \caption{Computational cost to reach specific reward scores (PS=22 or Aes=6). The gray region is the extra prompt-decomposing time with Real-time LLM. }
    \label{fig:time}
\end{figure}
\subsubsection{Two-factor episode length control}
After the end of an episode, the noise and semantics should be jointly considered in the generated image. We decide the episode length based on the $\alpha^d_t$ and $\alpha^s_t$. Firstly, $\alpha^d_t=1$ can be viewed as the denoising finish signal, and the corresponding denoising step $t^d_{\text{len}}$ is:
\begin{equation}
t^d_{\text{len}} = \min \left\{ t \mid \alpha^d_t \to 1 \right\}.
\end{equation}
For semantic alignment, we also include $\alpha^s_t$ and its variation. When both of them tend to 0, the semantic alignment is achieved. We use $t^s_{\text{len}}$ to represent the alignment step:
\begin{equation}
    t^s_{\text{len}} = \min \left\{ t \mid \alpha^s_t \to 0 \;\; \text{and} \;\; \left| \alpha^s_t - \alpha^s_{t+1} \right| \to 0 \right\}.
\end{equation}
Since the final episode length needs to satisfy both noise and semantic requirements, the episode step index is: 
\begin{equation}
    t_{\text{len}} = \min \left( t^s_{\text{len}}, \; t^d_{\text{len}} \right),
\end{equation}
and the episode length is $T-t_{len}$.

\section{Experimental Evaluation}

\subsection{Implementation Details}
\noindent \textbf{Datasets.} 
For fine-tuning, we introduce three datasets with rich and complex prompts, that is, HPSv2, Pick-a-Pic, and Simple animal. 
These prompts are then pre-processed before training by Chat-GPT-4o to obtain the \textit{Noun}-\textit{Attr} groups, respectively. For inference, the prompts are divided in real time.\\
\textbf{Rewards and Metrics.} In this paper, we introduce Aesthetic Score (AES)~\cite{aesthetic}, PickScore (PS)~\cite{kirstain2023pick}, JPEG compressibility, and incompressibility as the training reward functions. 
For the evaluation metrics, we deploy AES, PS, and ImageReward (IR)~\cite{xu2023imagereward} to assess the human preference and aesthetic impression, Clip Score (CLIP)~\cite{clip} for prompt-image alignment, and Inception Score (IS)~\cite{inception}, TCE~\cite{TCE}, and LPIPS~\cite{LPIPS} to assess image diversity.\\
\textbf{Baselines.} The primary criterion for evaluating a reinforcement learning (RL) fine-tuning algorithm is its ability to efficiently improve the target reward score. To this end, we compare our method against four state-of-the-art (SOTA) baselines (\textit{i.e.}, DDPO~\cite{black2023training}, DPOK\\~\cite{fan2024reinforcement}, D3PO~\cite{yang2024using}, and TDPO~\cite{zhang2024confronting}). Then, we deploy Stable Diffusion v1.5 (SDv15) as the main test bed, with Stable Diffusion v1.4 (SDv14),  v2.1-turbo (SDv21), XL1.0 (XL), and the flow-matching model SD3.5 as the supplementary backbones in experiments. All experiments were conducted on four NVIDIA Tesla H20 GPUs.

\begin{figure}[t]
    \centering
    \includegraphics[width=1\linewidth]{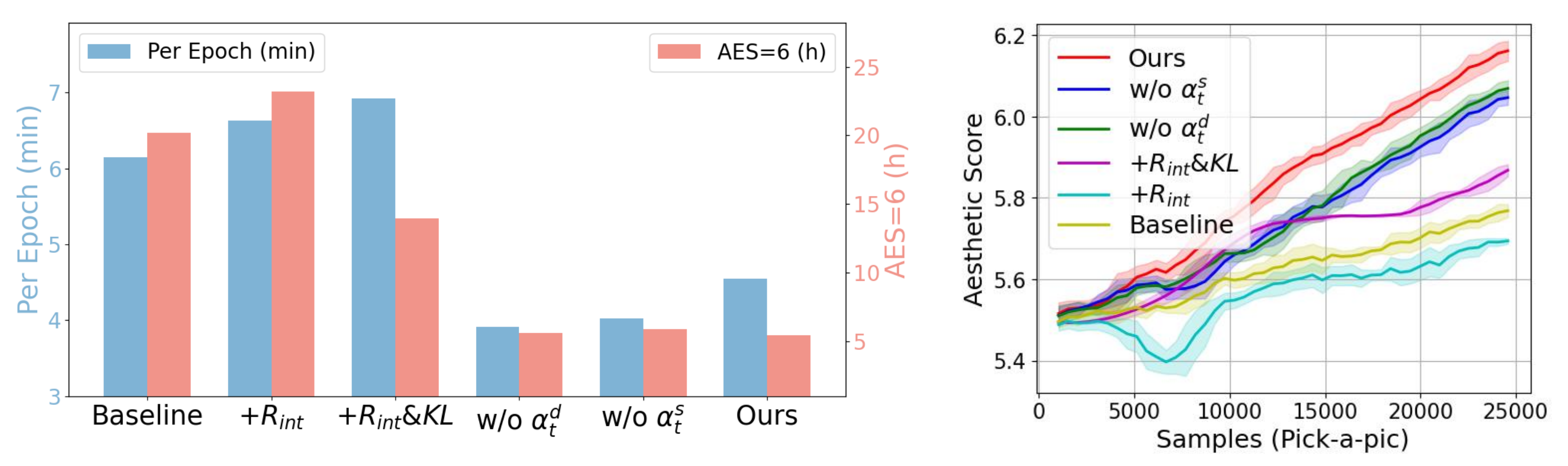}
    \caption{Ablation Study. \textbf{Left}: Computational cost per epoch. \textbf{Right}: Reward curves of each model configuration.}
    \label{fig:abl}
\end{figure}
\begin{table}[htbp]
\caption{Plug-and-Play Effectiveness. We report the time consumption (hours) for the Aesthetic Score to reach 6, and the aesthetic score with 2e4 training samples.}\label{tab:plug}
\small
\begin{tabular}{lcc}
\toprule
          & Time$^\downarrow$ for AES=6 & AES$^\uparrow$ for Sample=2e4 \\ \midrule
DDPO      &                6.67&                    5.64\\
DDPO+Ours &                2.22 (- 66.67\%)&                    5.90 (+ 4.61\%)\\ \midrule
DPOK      &                12.65&                    5.54\\
DPOK+Ours &                4.59 (- 63.71\%)&                    5.79 (+ 4.51\%)\\ \bottomrule
\end{tabular}
\end{table}

\begin{figure}[t]
    \centering
    \includegraphics[width=1\linewidth]{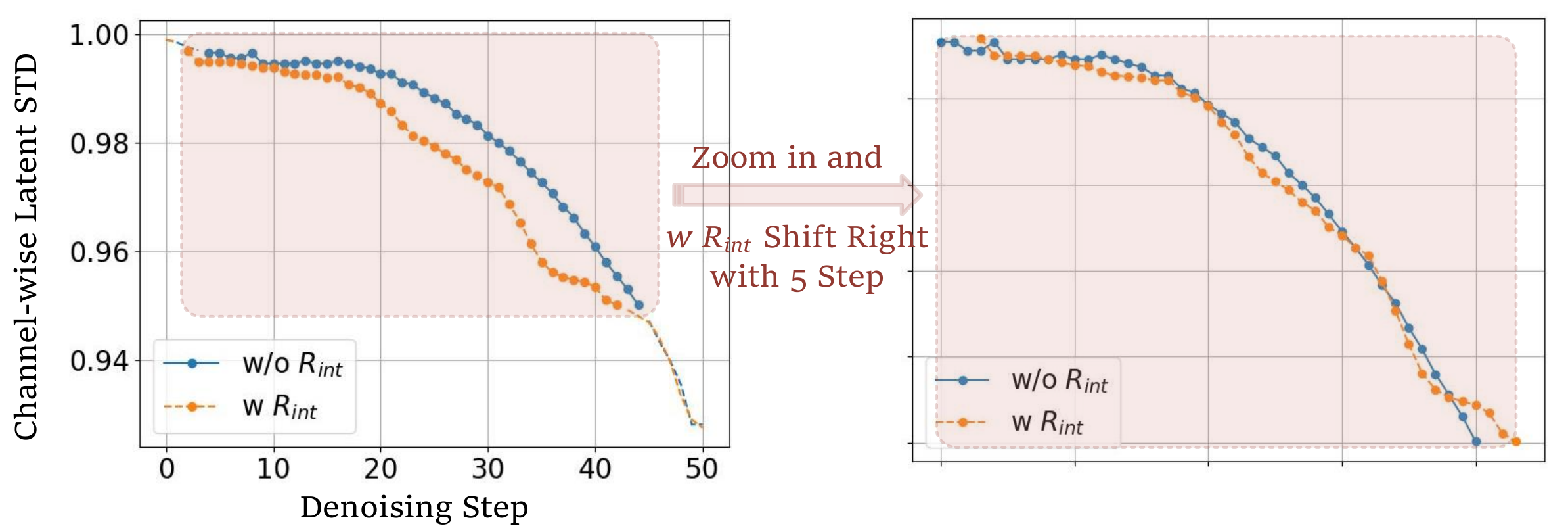}
    \caption{Visualization of channel-wise latent code for denoising acceleration. Standard deviation (STD) is employed to evaluate the state of the latent code across intermediate denoising steps. It can be observed that with $R_{int}$, the denosing steps is five less than the standard setting without $R_{int}$. }
    \label{fig:std}
    \vspace{-0.3cm}
\end{figure}
\begin{figure}[t]
    \centering
    \includegraphics[width=1\linewidth]{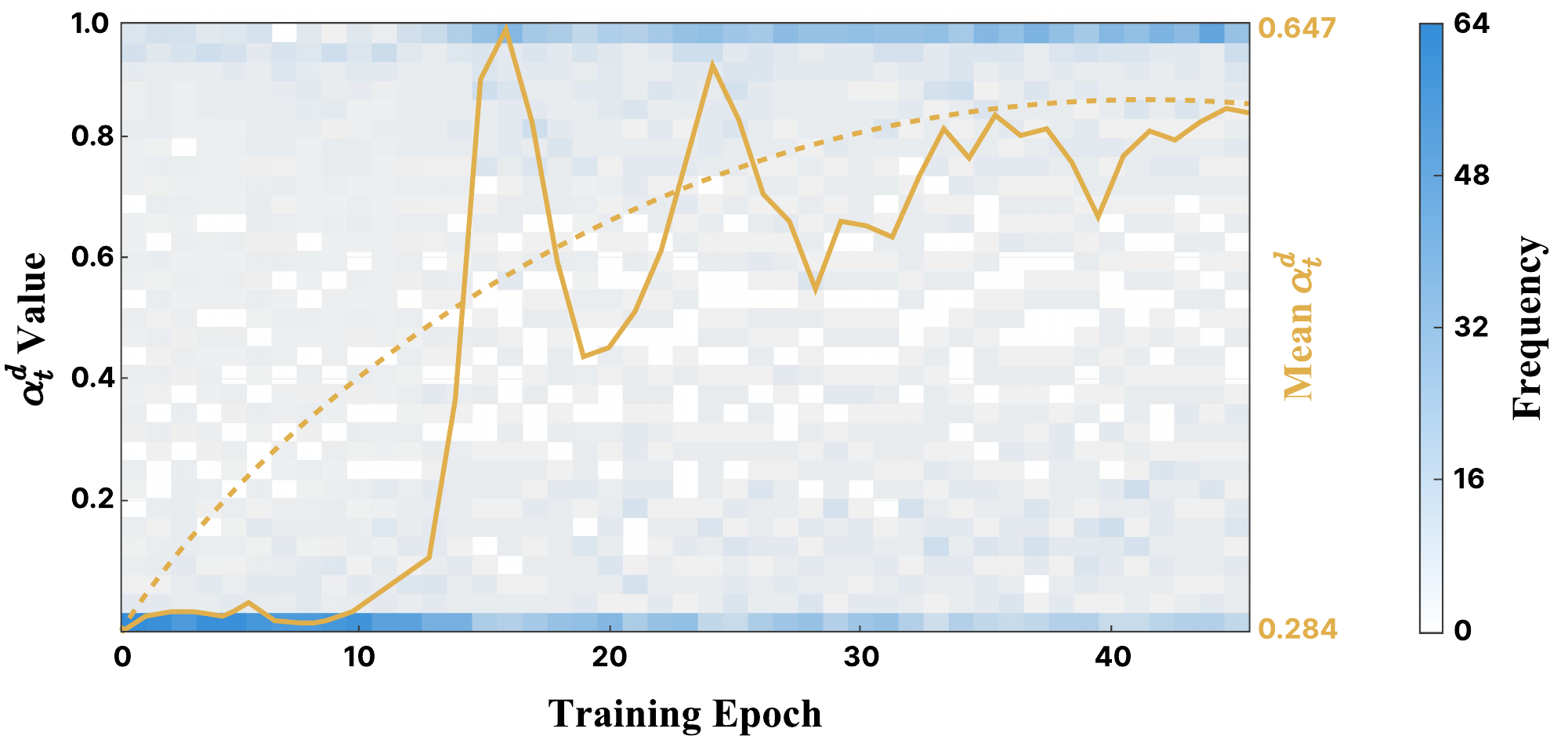}
    \caption{The distribution and trend of $\alpha^d_t$ from a frozen $f(\cdot)$ during fine-tuning with $R_{int}$. It can be observed that the distributions are gradually shifting toward 1, which indicates the denoising is accelerating during training. }
    \label{fig:predict}
\end{figure}
\subsection{Experiments on Training Acceleration}
\label{RLobj}
We compare the performance of \ourmethod{} with several reinforcement learning baseline methods for optimizing preset RL objectives (\textit{i.e.}, Aesthetic Score and PickScore). The experiments are conducted on SDv15, with Pick-a-pic and HPSv2 as the fine-tuning prompt set. More visual results and dataset can be found in Supplementary Material.
As shown in Fig.~\ref{fig:main}, our method consistently optimizes the target reward more efficiently under the same number of training samples, demonstrating higher training sample utilization. Additionally, we report the per-epoch training time of different methods on both HPSv2 and Pick-a-pic datasets, as well as the total time required to reach the predefined objective of AES = 6 or PickScore=22. As shown in Fig.~\ref{fig:time}, our method significantly outperforms existing approaches in optimization efficiency, while the extra computational cost introduced by LLM is rather trivial. Meanwhile, we found that the noun–attribute decomposition is trivial for modern models. In our experiments, both DeepSeek-V3 and Qwen2.5-32B-Instruct achieved 100\% consistency, indicating that this step can be performed reliably without introducing additional variance.

\subsection{Ablation Study}

For ablation experiments, we design the following variants: 1) Baseline: DDPO without any modification; 2) $+R_{int}$: include intrinsic reward only; 3) $+R_{int}\&KL$: apply $R_{int}$ and KL divergence; 4) w/o $\alpha^s_t$: include $R_{int}$, KL, and $\alpha^d_t$; 5) w/o $\alpha^d_t$: adopt $R_{int}$, KL, and $\alpha^s_t$. As shown in Fig.~\ref{fig:abl}-Right, the over-exploration of the intrinsic rewards causes fluctuations in the early training stage. Then, with the help of KL divergence, an exploration-exploitation balance is achieved. Overall, the results clearly show that as the number of components increases, the proposed method gradually exhibits improved optimization performance. We provide further ablations on the influence of KL and $R_{\text{int}}$ in Appendix Fig~\ref{fig:sd15_abl_rebuttal}.
The computational results in Fig.~\ref{fig:abl}-Left indicate the effectiveness of the adaptive episode control strategy, while considering both factors to determine the episode length could achieve the best reward learning efficiency.

\begin{figure}[t]
    \centering
    \includegraphics[width=1\linewidth]{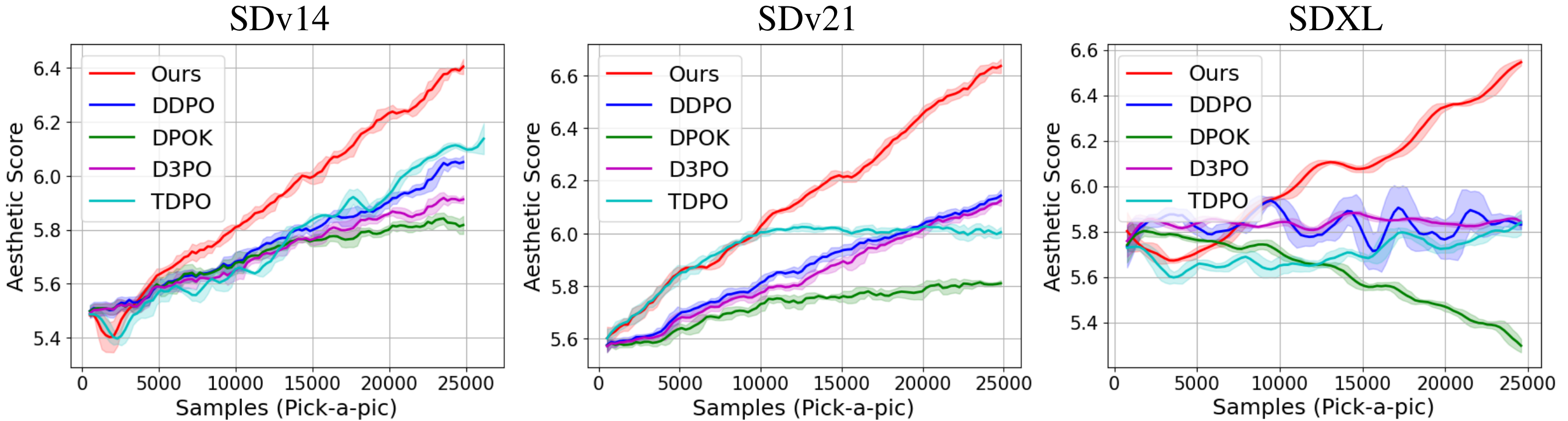}
    \caption{Transferring to different diffusion backbones.}
    \label{fig:backbones}
\end{figure}
\subsection{Analysis of PAST}
\noindent \textbf{Denoising acceleration.} 
In DDIM, the generated images are decoded from the noisy-to-clean latent code. To verify the denoising acceleration of introducing $R_{int}$, we first compare the differences in the latent code distribution before and after fine-tuning with $R_{int}$. Although the latent code distribution for each denoising step is determined, the channel-wise latent code distribution is not explicitly constrained. Therefore, we analyze the channel-wise standard deviation (std) of the latent code. As shown in Fig.~\ref{fig:std}, the std of w $R_{int}$ is closer to that of the future steps compared with w/o $R_{int}$. For a clearer illustration, by shifting the $R_{int}$ curve, it closely overlaps with the curve of w/o $R_{int}$ after five steps. This further demonstrates that using $R_{int}$ enables the denoising state, \textit{i.e.}, latent code, to approach the real distribution faster, thus highlighting the acceleration effect of the intrinsic reward on denoising.

\begin{figure*}[ht]
    \centering
    \includegraphics[width=1\linewidth]{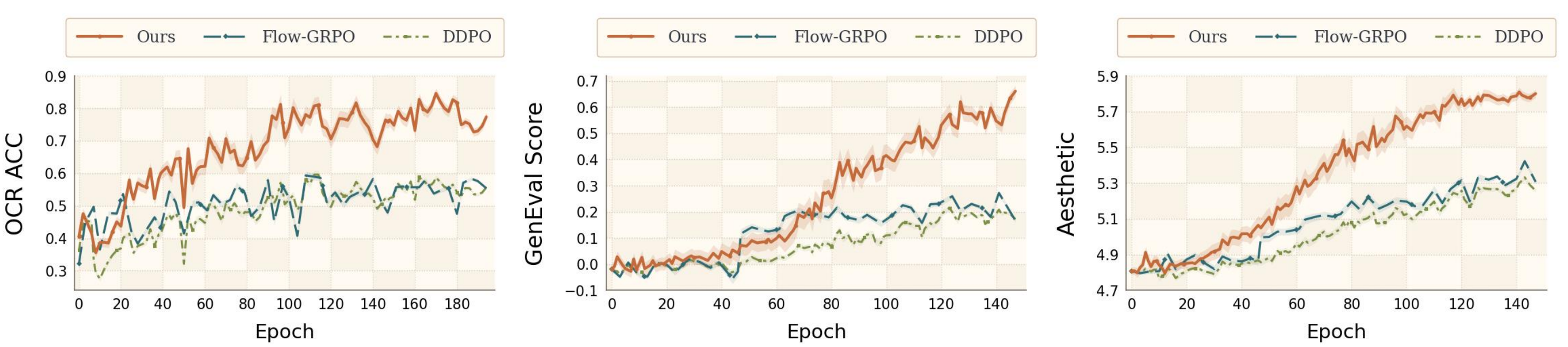}
    \caption{Performance on a \underline{flow-matching} model (SD3.5). Results on OCR, GenEval, and Aesthetic objectives show that \ourmethod{} remains consistently effective under diverse evaluation targets.}
    \label{fig:flow-backbone}
    \vspace{-0.4cm}
\end{figure*}
\begin{table}[t]
\caption{Analysis on episode length control strategies.}\label{tab:step}
\centering

    \setlength{\tabcolsep}{4.7pt}  
\begin{tabular}{lccccccc}
\toprule
Method  & AES  & PS    & IR   & CLIP  & IS    & TCE   & LPIPS \\\midrule
DDPO-F & 6.06 & 19.74 & 0.76 & 0.286 & 13.50 & 37.16 & 0.580 \\
DDPO-C   & 6.50 & 21.72 & 1.00 & 0.307 & 13.97 & 39.29 & 0.579 \\
Ours-F & 6.17 & 20.89 & 0.79 & 0.294 & 13.69 & 38.41 & 0.581 \\
Ours-C   & 6.50 & 21.84 & 1.01 & 0.312 & 14.12 & 39.66 & 0.587 \\ \midrule
Ours-A & 6.47 & 22.27 & 0.98 & 0.313 & 14.09 & 40.16 & 0.593 \\ \bottomrule
\end{tabular}
\vspace{-0.2cm}
\end{table}
For denoising effectiveness, we design the following experiment: Considering that $\alpha^d_t$ predicts the degree of denoising, we introduce a pre-trained and frozen noise predictor to monitor the fine-tuning process with $R_{int}$. As shown in Fig.~\ref{fig:predict}, for each iteration with 50 denoising steps completed, the distribution of $\alpha^d_t$ gradually shifted towards 1 as training progressed, with the mean increasing from 0.3 to 0.6. This further demonstrates that the denoised samples successfully appear earlier in the denoising process.

\begin{figure*}[t]
    \centering
    \includegraphics[width=1\linewidth]{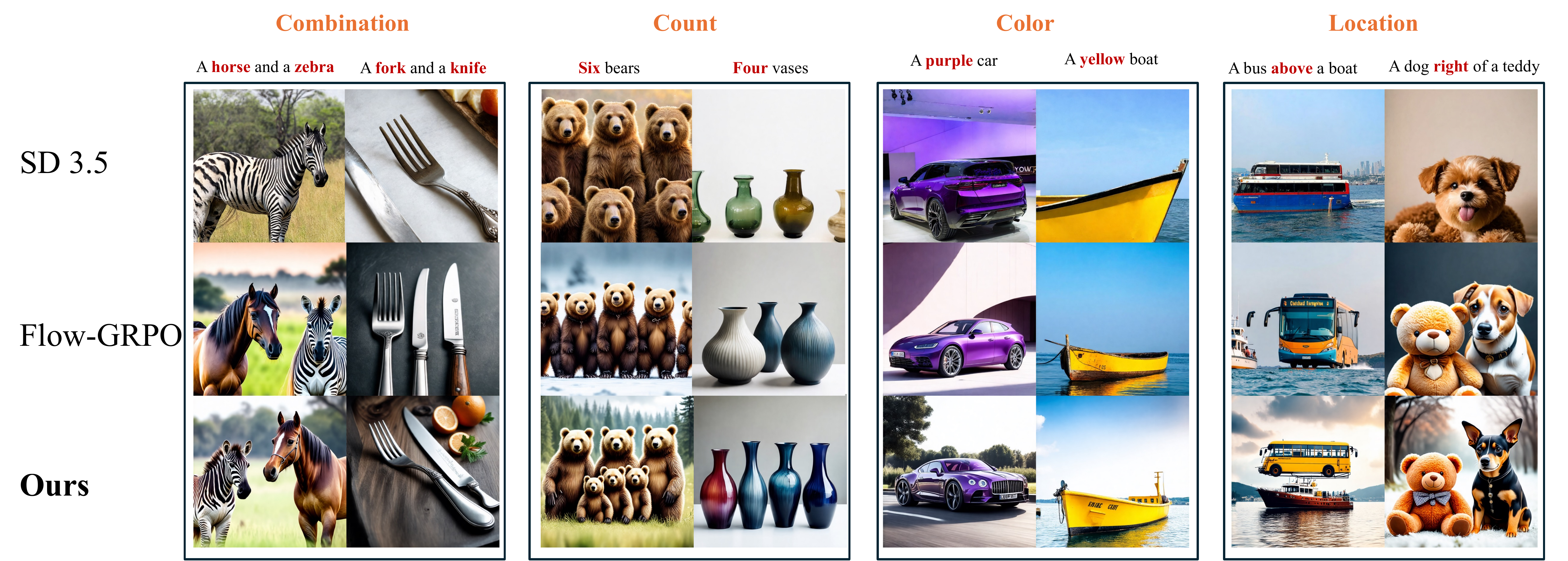}
        \vspace{-0.3cm}
    \caption{Conditional Generation on the \underline{flow-matching} Diffusion Model (SD3.5). Compared to baseline methods, \ourmethod{} better follows instructions when prompts explicitly define combinations, counts, colors, and locations.}
    \label{fig:past_35_align}
    \vspace{-0.2cm}
\end{figure*}

\begin{figure}[t]
    \centering
    \includegraphics[width=1\linewidth]{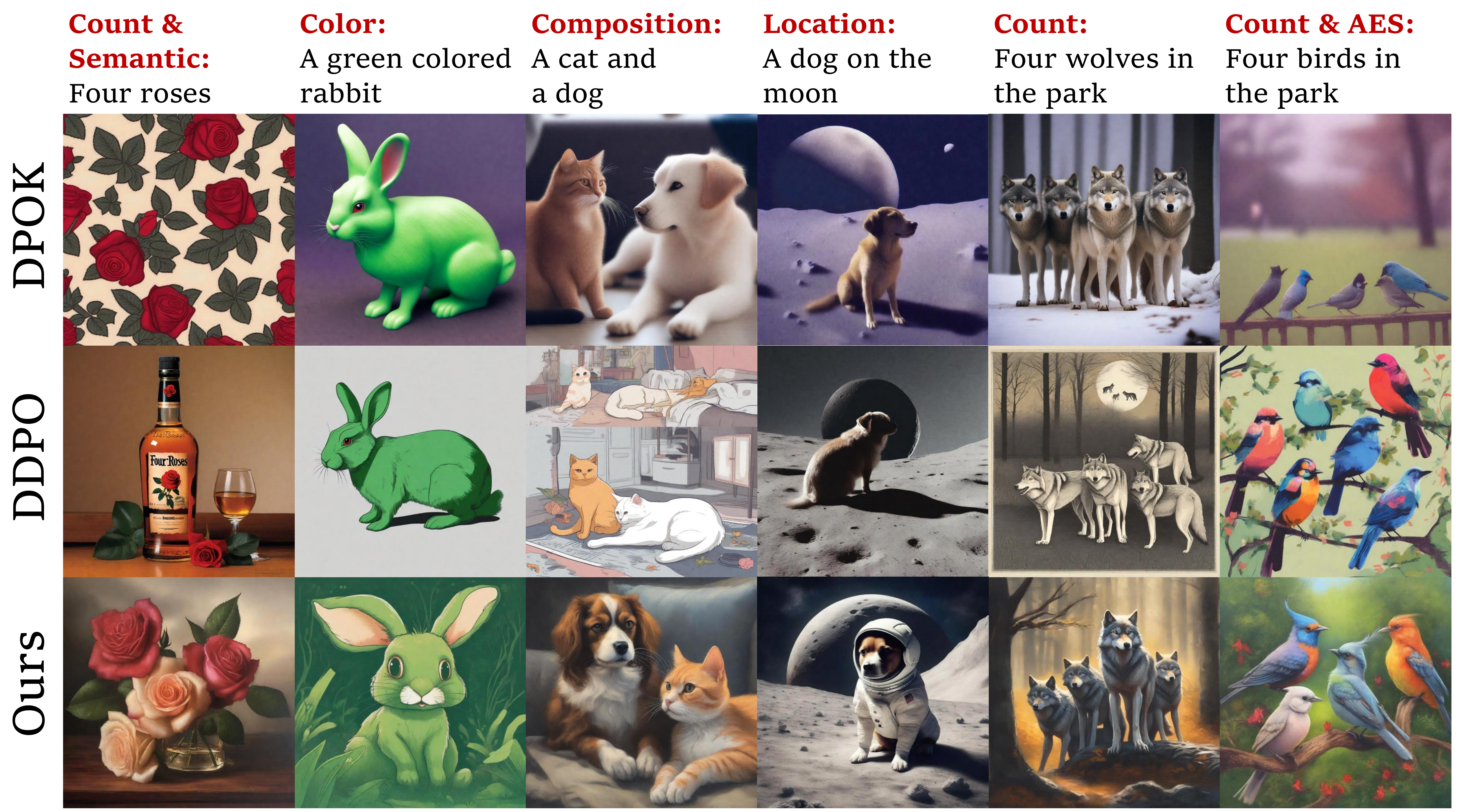}
    \caption{Semantic alignment on unseen prompts. We exhibit superior alignment in multiple semantic content, including Count, Color, Composition, and Location.}
    \label{fig:vis-align}
    \vspace{-0.75cm}
\end{figure}

\noindent \textbf{Two-factor episode length control.}
Here, we evaluate the cases of adaptive episode, complete episode, and fixed episode during the inference phase, with image quality among the following conditions: Ours with adaptive episode (Ours-A), Ours with fixed episode (Ours-F), Ours with complete episode (Ours-C), DDPO with fixed episode (DDPO-F), and DDPO with complete episode (DDPO-C). For a fair comparison, we exhibit the results when the AES score of DDPO and Ours reaches 6.5.
Since the average number of steps for adaptive episode is 31, the fixed episode method was set to 31 steps. In Tab.~\ref{tab:step}, we first observe that using fixed episode results in a significant drop in metrics for both DDPO and Ours. However, with the adaptive episode scheme, the same average denoising length achieves results comparable to those with complete denoising. Furthermore, with identical aesthetic scores for the corresponding original models (\textit{i.e.}, DDPO-C and Ours-C), Ours-F outperforms DDPO-F in all metrics. This further suggests that \ourmethod{} accelerates the denoising process and reaches the true distribution earlier.\\
\noindent \textbf{Plug-and-Play effectiveness.}
Considering \ourmethod{} is model-agnostic, we deploy our method to the popular DDPO and DPOK. In Tab.~\ref{tab:plug}, adding our method can effectively reduce the overall computational cost and enhance the reward learning performance. Further elaboration on efficiency is available in Appendix.~\ref{Efficiency}.

\subsection{Transfer Study}
\noindent \textbf{To other backbones.} Here, we conduct experiments on SDv1.4, SDv2.1, and the more advanced SD-XL. As shown in Fig.~\ref{fig:backbones}, the superior results highlight that \ourmethod{} is not restricted to a specific backbone, but could be directly transferred to a wide range of backbone architectures.
Additionally, we evaluate \ourmethod{} on a flow-matching model (SD3.5). As shown in Fig.~\ref{fig:flow-backbone}, \ourmethod{} consistently improves performance on OCR, GenEval, and Aesthetic objectives, demonstrating strong generalization across both model paradigms and heterogeneous reward criteria.\\
\noindent \textbf{To other reward objectives.} Additional transfer results on different reward functions are provided in Appendix~\ref{appendix:reward_transfer}.\\

\subsection{Visual Quality}
Here, we evaluate fine-grained semantic alignment under conditional generation scenarios on both SD3.5 (Fig.~\ref{fig:flow-backbone}) and SDXL (Fig.~\ref{fig:vis-align}), confirming that we maintain robust consistency across semantic dimensions including count, color, composition, and location. This result demonstrates that explicitly leveraging prompt–image attention to facilitate model training, along with timely truncation, can endow the post-trained model with enhanced semantic alignment capability.
Additional visual comparisons on more diverse prompts, extended unseen-prompt evaluations, and generative diversity analyses can be found in Appendix Fig.~\ref{fig:vis-main}, \ref{fig:vis-unseen}, and \ref{fig:div_all_past}.

\begin{figure}[htbp]
    \centering
    \includegraphics[width=0.9\linewidth]{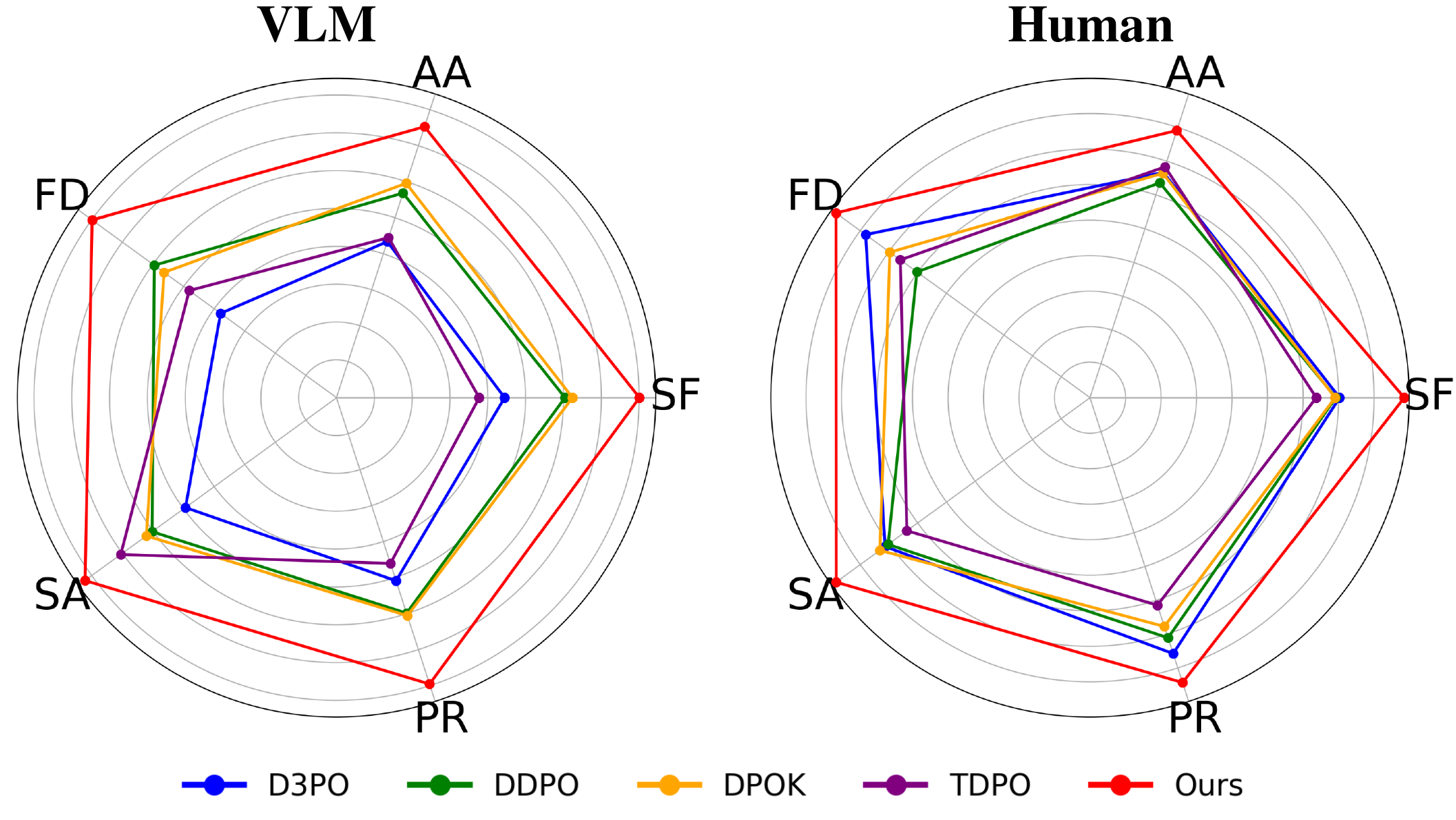}
    \caption{Subjective Evaluation on Human and VLM.}
    \label{fig:user}
    \vspace{-0.4cm}
\end{figure}
\subsection{Subjective Evaluation}
Subjectively, we employed two assessments: Human scoring and Large Vision Model (LVM) scoring. All images were generated by the fine-tuned SDv15 model from HPSv2 prompts. The evaluation covered five dimensions—Structural Faithfulness (SF), Aesthetic Appeal (AA), Fine-grained Detail (FD), Semantic Alignment (SA), and Prompt Responsiveness (PR)—rated independently by humans and ChatGPT. Additional evaluation details can be found in the \textit{Supplementary Material}. In Fig.~\ref{fig:user}, \ourmethod{} consistently outperforms alternatives across both assessment approaches and metrics.

\section{Conclusion}
In this paper, we propose \ourmethod{}, an efficient plugin for RL fine-tuning of diffusion models, which involves an adaptive episode control to enable a broader search range for high reward policy preserving real distribution. By introducing fine-grained intrinsic rewards and an adaptive trade-off strategy, it effectively alleviates the issues of sparse rewards and high training cost. Experimental results show that our method significantly reduces computational overhead while improving downstream performance.

\clearpage

\bibliographystyle{ACM-Reference-Format}

\bibliography{refer.bib}

\clearpage

\appendix

\clearpage

\onecolumn

\section{Additional Related Work}
\label{appendix:related_work}

\subsection{RL Fine-tuning for Diffusion Models}
Traditional likelihood-based training objectives can lead to a decline in image quality, as likelihood values often fail to reflect visual quality~\cite{black2023training} accurately. Reinforcement learning (RL) methods build on pre-trained text-to-image diffusion models, further fine-tuning them to directly optimize specific downstream task objectives, making generated images better aligned with task requirements. Among these, reward model-based methods, such as DDPO~\cite{black2023training} and DPOK~\cite{fan2024reinforcement}, optimize generation quality by constructing reward functions to evaluate model outputs. On the other hand, human feedback-based methods, such as D3PO~\cite{yang2024using}, Diffusion-DPO~\cite{wallace2024diffusion}, and HPS~\cite{wu2023human}, leverage human preference scores directly, avoiding the complexity of reward modeling and aiming to align generation more closely with human aesthetic preferences. However, these approaches all face issues with delayed and sparse rewards or feedback~\cite{franceschelli2024reinforcement} and the balance between exploration and exploitation~\cite{jena2024elucidating}.
~\textbf{Although existing RL fine-tuning methods demonstrate significant advantages in improving the generation quality and target alignment of diffusion models, the additional computational overhead they introduce remains non-negligible and has become a key obstacle to their practical deployment~\cite{kim2025test}}.

\subsection{Intrinsic Motivation}
In reinforcement learning tasks, intrinsic rewards~\cite{pathak2017curiosity,burda2018exploration} are commonly used to guide agents in effective exploration when extrinsic feedback is lacking or rewards are sparse. Intrinsic rewards are independent of extrinsic environment rewards, providing reward signals based on the agent's exploratory behavior or interaction states in intermediate steps. This encourages more diverse exploration within the task space, preventing training stagnation or low learning efficiency caused by insufficient extrinsic rewards. Intrinsic rewards bypass direct evaluation of specific real rewards on intermediate generated images and instead focus on exploration of intermediate states, agnostic to particular optimization objectives. Introducing intrinsic rewards usually does not alter the original optimal strategy under extrinsic rewards~\cite{yan2024adazero}. However, a well-designed intrinsic reward mechanism for diffusion models is crucial, as existing intrinsic rewards are based on the agent's trial-and-error settings and cannot be directly applied to diffusion models. By fully considering the generative nature of diffusion models, this work devises a general-purpose intrinsic reward scheme specifically customized for their optimization to overcome the reward sparsity issue in optimizing diffusion models.

\subsection{Reward Hacking during RL Fine-tuning}
In the RL fine-tuning of T2I tasks, reward hacking also remains a non-negligible issue, where the model maximizes the reward score but generates outputs that lack diversity or authenticity~\cite{skalse2022hacking,jena2024elucidating}. For example, when the model overly prioritizes specific reward signals (such as image aesthetics or text alignment), it can lead to structural distortions, loss of detail, or uniform patterns in generated images, ultimately compromising generation quality. 
Reward hijacking has long plagued RL-enhanced diffusion models and remains unresolved. In this work, we introduce intrinsic rewards to encourage exploration in intermediate stages, providing fine-grained optimization signals during denoising. This alleviates over-reliance on the sparse and delayed final rewards and significantly reduces the risk of hijacking. We validate this mechanism through visualization experiments (see Fig. ~\ref{fig:hacking}), showing that even intrinsic rewards alone can effectively mitigate reward hijacking.

\section{Supplementary Experimental  Details and Analysis}
\subsection{Experiment List in Our Paper}
To help readers quickly grasp the extensive experiments conducted in this work, we summarize the full list of experiments below.
\begin{itemize}
    \item \textbf{(1) Visualization Experiments.} See Fig.~\ref{fig:intro}. This experiment provides a solid justification for the motivation of this work.

    \item \textbf{(2) Visualization Experiments.} See Fig.~\ref{fig:std}. This experiment provides a promising demonstration of the denoising acceleration achieved by our method.

    \item \textbf{(3.1) Impression of Reward Hacking} See Tab.~\ref{fig:hacking}. We provide a visual impression for a better understanding of reward hacking, and also the performance of introducing intrinsic reward to mitigate reward hacking.

    \item \textbf{(3.2) Ensemble Experiments.} See Tab.~\ref{tab:plug}. This table further shows the performance gains obtained by incorporating our RL-based enhancement plug-in into existing approaches.

    \item \textbf{(4) Dataset Switching Experiment.} See Fig.~\ref{fig:main}. To verify that the superiority of our method is not affected by differences in datasets or task difficulty, we conduct experiments across multiple datasets.

    \item \textbf{(5) Backbone Switching Experiment.} See Fig.~\ref{fig:backbones}. We provide this experiment to verify that the advantages of our method are not affected by replacing the backbone to be fine-tuned.

    \item \textbf{(6) Reward Switching Experiment.} See Fig.~\ref{fig:objectives} and \ref{fig:main}. We verify that the advantages of our method are not affected by replacing the reward model.

    \item \textbf{(7.1) Ablation Experiment.} See Fig.~\ref{fig:abl}. We present the impact of each module in our method on both performance and computational cost.

    \item \textbf{(7.2) Effect of different episode length control strategies.} See Fig.~\ref{fig:predict}. We present the impact of different episode length control strategies on diverse evaluation metrics.

    \item \textbf{(8) Computational Efficiency.} See Fig.~\ref{fig:time}. In this experiment, we comprehensively demonstrate the superior computational cost of our method.

    \item \textbf{(9) Semantic Alignment Visual Experiment.} See Fig.~\ref{fig:vis-align}. This experiment verifies the semantic alignment capability of our method.

    \item \textbf{(10) Unseen-Prompt Generalization.} See Fig.~\ref{fig:vis-unseen}. This experiment verifies our method’s ability to handle complex prompts.

    \item \textbf{(11) Diversity Visual Experiment.} See Fig.~\ref{fig:div_all_past}. This experiment compares our method with state-of-the-art approaches in terms of diversity. Validating our claim that it mitigates reward hacking and prevents excessive diversity degradation during fine-tuning.

    \item \textbf{(12) Human Evaluation Experiment.} See Fig.~\ref{fig:user}. We present a human evaluation experiment assessing our method against the baseline methods.

\end{itemize}

\subsection{Supplementary Analysis on Optimization Efficiency}
\label{Efficiency}
Experimental results show that \ourmethod{} significantly improves the overall efficiency of diffusion model fine-tuning, reflected in two aspects: \textbf{(1) improved training time efficiency} (see Fig.~\ref{fig:time}; Fig.~\ref{fig:std}; Fig.~\ref{fig:predict}), and \textbf{(2) higher reward acquisition efficiency and final reward level} (see Fig.~\ref{fig:main}; Fig.~\ref{fig:backbones}; Fig.~\ref{fig:objectives}).

First, the improvement in \textbf{time efficiency} mainly comes from the acceleration of the denoising process by the intrinsic reward mechanism, and the adaptive episode length scheduling strategy built upon it. This strategy actively identifies and discards inefficient samples during the marginal gain phase, effectively reducing redundant training steps from the perspective of reinforcement learning policy optimization, thereby lowering computational and time costs.

Second, the improvement in task \textbf{reward acquisition efficiency} mainly comes from the joint effect of two key mechanisms: a dynamic dual trade-off strategy and an adaptive episode termination mechanism.
Previous studies have shown that overemphasizing exploration while ignoring exploitation may lead to reduced training efficiency and even harm the exploration policy itself; whereas excessive reliance on exploitation without sufficient exploration may cause the policy to fall into local optima, missing the potential for near-global optimality~\cite{taiga2021bonus,dabney2020temporally}.
Therefore, building a proper exploration–exploitation trade-off mechanism is essential for improving policy quality.
Based on this motivation, we propose the following dual trade-off strategy:
\begin{itemize}
\item Dynamic coordination between intrinsic and extrinsic rewards
\item Balancing novel trajectory exploration and stable convergence to the real distribution (via KL regularization)
\end{itemize}

Through these two trade-off mechanisms, \ourmethod{} enables dynamic exploration–exploitation coordination at different stages: in the early generation phase, the policy encourages exploration to discover more potential high-quality generation paths; in the later phase, it gradually shifts to stable exploitation, focusing on high-reward regions within the real image distribution (guaranteed by KL regularization), thereby improving final generation quality and task performance. This is the essential reason for the higher reward curves of \ourmethod{}. In the ablation study shown in Fig.~\ref{fig:abl}, both w/o $\alpha^d_t$ and w/o $\alpha^s_t$ perform worse than the full \ourmethod{}, further confirming the positive effect of the dynamic trade-off on reward improvement.



\subsection{Plug-and-Play Experiment Analysis}
As shown in Tab.~\ref{tab:plug}, our method is plug-and-play because it requires no modification to any components or strategies of the base approach. Instead, it adds an intrinsic reward to accelerate denoising and uses adaptive episode control to regulate the optimization length of the base method itself. The introduced components effectively identify redundant samples and prevent them from interfering with fine-tuning, thereby achieving strong plug-and-play acceleration performance.

\subsection{Hyper-parameters}

The main hyperparameters in our paper is shown in Table~\ref{tab:hyper}.
For each method and each RL objective, we ran five different seeds and report the mean and standard deviation of reward on 64 randomly sampled prompts as validation set. We train each model with a total of 25000 samples. Following~\cite{black2023training}, we use the LAION aesthetics predictor for conducting the aesthetics experiments.

\subsection{Implementation details of perception model}

In Subsection~\ref{sec:21}, we introduced a Denoise-Aware Model $f(\cdot)$ that learned to predict the dynamic coefficient of the intrinsic reward according to how close the image in the current de-noising step is to the fully de-noised real images. The architecture of the perception model is a simple four-layer CNN followed by a fully connected layer and a Sigmoid layer for binary classification. The kernel sizes are (3, 3, 3, 1), and channels are (64, 128, 256, 256).

To train the perception model, we update its parameters after each optimization step of the RL loss, with data from the current de-noising chain. To construct the training data, given a $T$-step de-noising process, we take the first $n$ images near $\{\mathbf{x}_T\}$ as negative samples and the last $n$ images $\{\mathbf{x}_0\}$ as positive samples, where $n=10$.
In this manner, we set the intrinsic rewards to zero for almost real images, and also obtain a denoise evaluator for adaptive episode length control.

\subsection{Subjective Evaluation Details}

For a more accurate evaluation of the models after fine-tuning for alignment, we conducted subjective evaluations. As shown in Fig.~\ref{fig:Interface}, we employed the user interface to collect human feedback. This interface facilitated the collection of 1296*5 valid data points from multiple methods in the user experiment, and the mean values were calculated for analysis and presentation.

For Large Vision Model evaluation, we deploy ChatGPT-4o with the following definitions of each attribute:
\begin{itemize}
    \item \textbf{Semantic Alignment}: Does the generated image accurately and comprehensively convey the key entities, scenes, and actions described in the text? (Measures the degree of direct correspondence between the textual content and the visual output.)
    
    \item \textbf{Structural Faithfulness}: Does the image exhibit any structural anomalies, such as disproportionate elements, extra limbs, or misaligned backgrounds? (Assesses the logical coherence and structural plausibility of the visual composition.)
    
    \item \textbf{Aesthetic Appeal}: Which image demonstrates superior appeal in terms of visual style, composition, color harmony, and overall aesthetics? (Reflects traditional notions of visual attractiveness.)
    
    \item \textbf{Fine-grained Detail}: Which image exhibits greater finesse and naturalism in rendering textures, materials, shadows, and other fine-grained visual details? (Reflects perceptual resolution and detail fidelity.)
    
    \item \textbf{Prompt Responsiveness}: Which image more accurately reflects the attribute constraints specified in the prompt, such as color, quantity, or action? (Used to evaluate the controllability and precision of prompt adherence.)

\end{itemize}

We first feed the above definitions to the LVM. Then, we upload the images to LVM and obtain the final rating.

\section{Supplementary Experimental Results}
\subsection{Transfer to Different Reward Functions}
\label{appendix:reward_transfer}
For transferability across different reward objectives, we further introduce additional rewards, including JPEG compressibility, incompressibility, and AES combined with PickScore (AES+PS), to perform multi-objective evaluations. In Fig.~\ref{fig:objectives}, the results on pick-a-pic and SDv14 demonstrate that \ourmethod{} consistently exhibits superior performance across all criteria.

\subsection{Detailed Ablation Study}
Here, we further analyze the weighting strategy using the SD1.5 ablation setting shown in Fig.~\ref{fig:sd15_abl_rebuttal}. The $R_{\text{int}}$-only variant encourages faster denoising exploration, but excessive exploration can induce collapse. In contrast, the KL-only variant maintains stable convergence yet tends to weaken preference improvement. Fixed-weight combinations partially inherit the benefits of both terms, but remain sensitive to manually selected coefficients; the shadowed curves illustrate this limitation across multiple fixed values. Our adaptive weighting resolves this trade-off by emphasizing $R_{\text{int}}$ in the early stage and progressively shifting toward preference reward with stable convergence in later stages, leading to stronger and more stable preference optimization. This behavior is consistent with our design principle in Sec.~\ref{sec:trade-off}: exploration becomes more effective when regulated by an adaptive mechanism.

\begin{figure}[t]
    \centering
    \begin{minipage}[t]{0.77\linewidth}
        \centering
        \includegraphics[width=\linewidth]{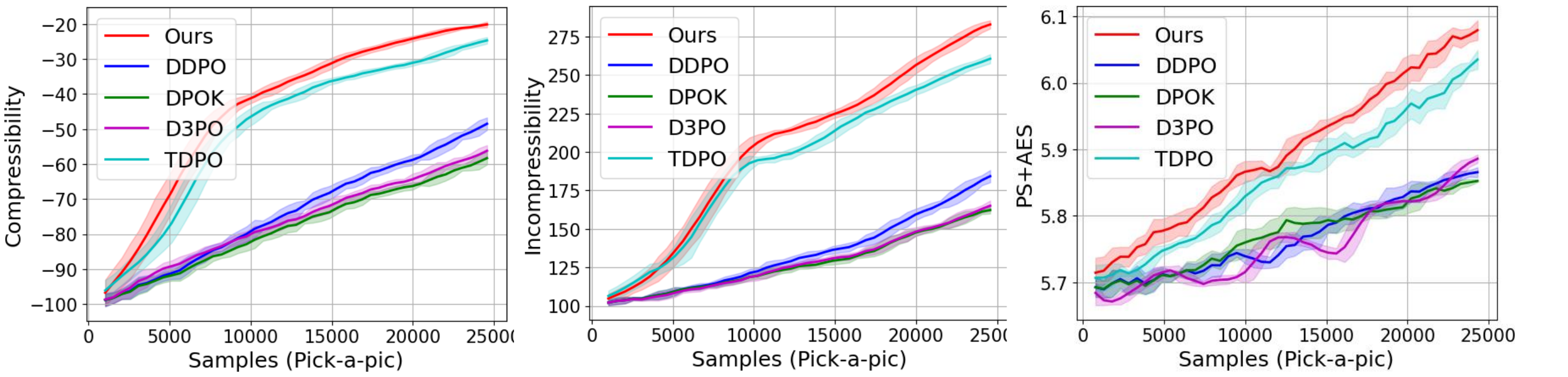}
        \captionof{figure}{Transferring to different reward functions.}
        \label{fig:objectives}
        \label{fig:objectives_transfer}
    \end{minipage}\hfill
    \begin{minipage}[t]{0.20\linewidth}
        \centering
        \includegraphics[width=\linewidth]{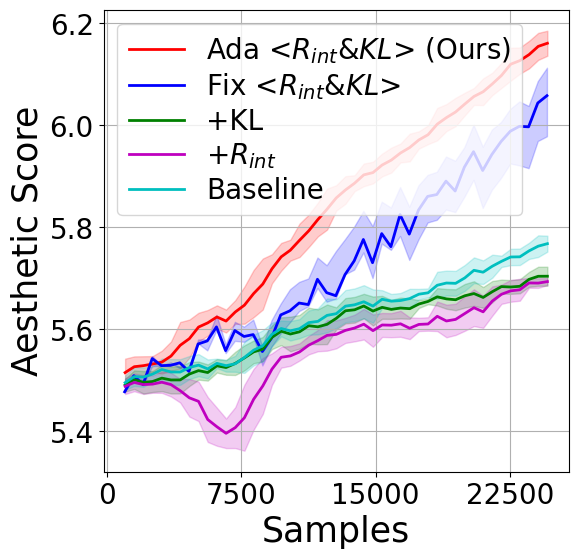}
        \captionof{figure}{Detailed ablation on KL and $R_{\text{int}}$}
        \label{fig:sd15_abl_rebuttal}
    \end{minipage}
\end{figure}
\subsection{Further Visual Quality Evaluation}
\noindent\textbf{Aesthetic quality.} Since AES is treated as the reward objective for fine-tuning, we present generated images to discuss the aesthetic visual quality with the training Pick-a-pic prompt set and SDv15. As shown in Fig.~\ref{fig:vis-main}, our method demonstrates enhanced details and a more impressive aesthetic quality while preserving semantic information.\\
\noindent\textbf{Visual quality on unseen prompts.} To demonstrate the generalization capability of the fine-tuned model, we conducted evaluations on SDxl using a set of prompts unseen during training. First, in Fig.~\ref{fig:vis-unseen}, we conduct visual experiments to discuss overall image quality. It can be observed that our method achieves superior image structure and aesthetic quality even with unseen prompts.\\
\noindent\textbf{Generative Diversity}
Finally, in Fig.~\ref{fig:div_all_past}, we show the generative diversity of different methods. It can be observed that our results achieve the best aesthetic quality while precisely maintaining the text-image alignment.

\begin{figure}[t]
    \centering
    \includegraphics[width=1\linewidth]{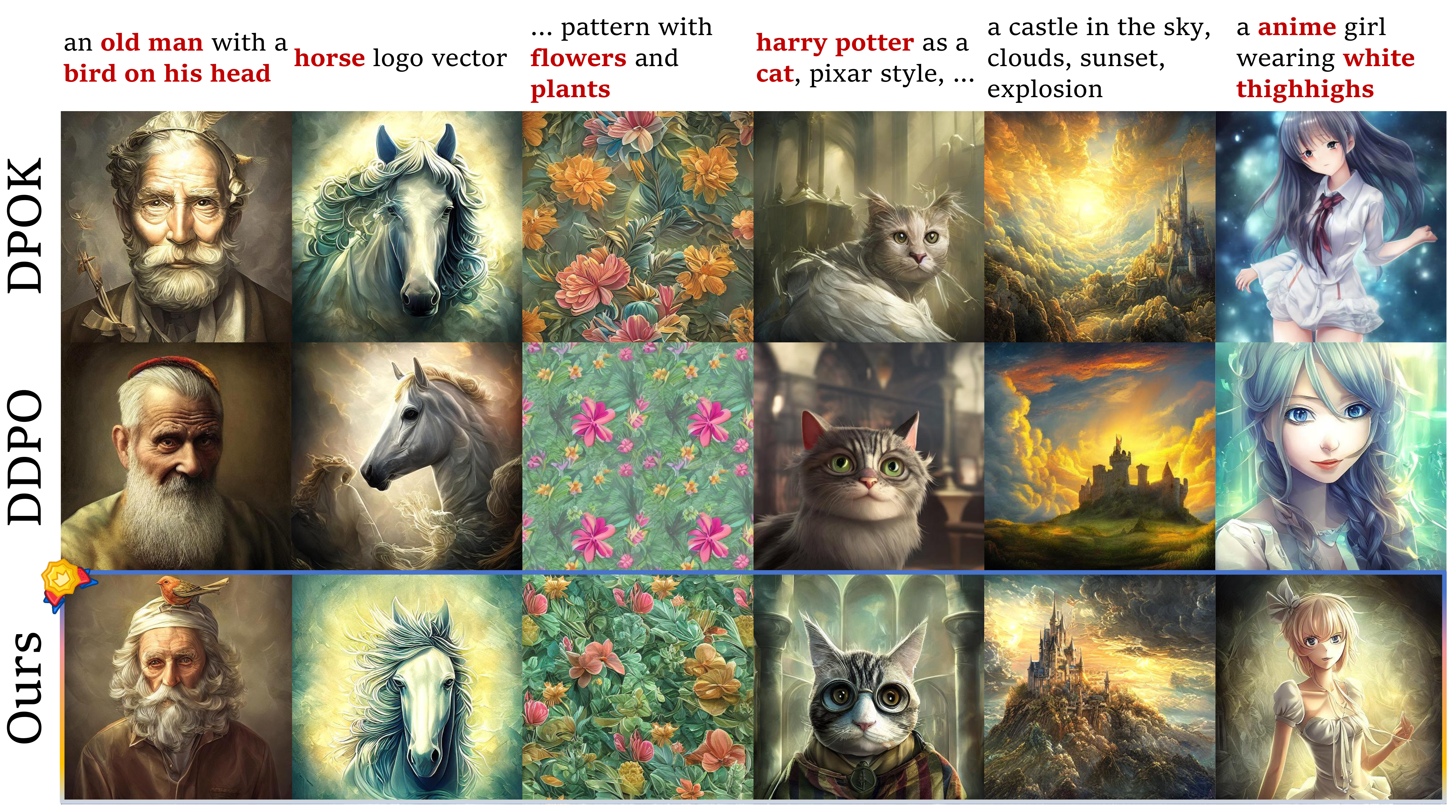}
    \caption{Aesthetic quality of generated images. }
    \label{fig:vis-main}
\end{figure}
\begin{figure}[t]
    \centering
\includegraphics[width=1\linewidth]{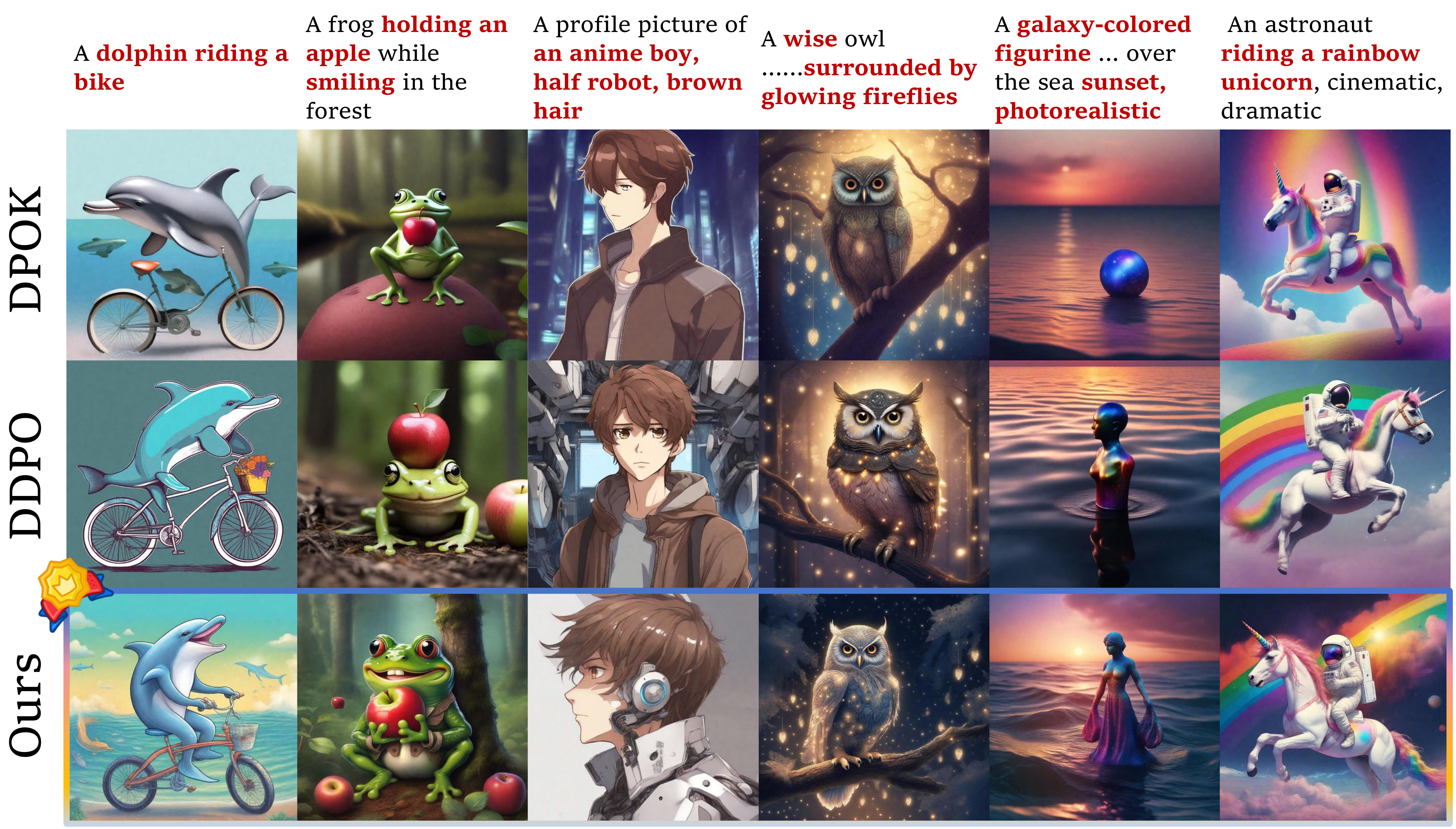}
    \caption{Overall visual quality on unseen prompts.}
    \label{fig:vis-unseen}
\end{figure}




\begin{table}[h!]
\centering
\begin{tabularx}{\linewidth}{|l|X|l|}
\hline
\textbf{Name} & \textbf{Description} & \textbf{Value} \\ \hline
\textit{lr} & learning rate of \ourmethod{} & 3e-4 \\ \hline
optimizer & type of optimizer & Adam \cite{kingma2014adam} \\ \hline
$\xi$ & weight decay of optimizer & 1e-4 \\ \hline
$\epsilon$ & Gradient clip norm & 1.0 \\ \hline
$\beta_1$ & $\beta_1$ of Adam & 0.9 \\ \hline
$\beta_2$ & $\beta_2$ of Adam & 0.999 \\ \hline
$T$ & total timesteps of inference & 50 \\ \hline
$bs$ & train batch size per GPU & 2 \\ \hline
$bs_{sample}$ & sample batch size per GPU & 16 \\ \hline
$n$ & number of batch samples per epoch & 4 \\ \hline
$\eta$ & eta parameter for the DDIM sampler & 1.0 \\ \hline
$G$ & gradient accumulation steps & 4 \\ \hline
$w$ & classifier-free guidance weight & 5.0 \\ \hline
$N$ & epochs for fine-tuning with reward model & 100 \\ \hline
$mp$ & mixed precision & fp16 \\ \hline
\end{tabularx}
\caption{Hyper-parameters in our experiment.}
\label{tab:hyper}
\end{table}

\begin{figure*}
    \centering
    \includegraphics[width=0.99\linewidth]{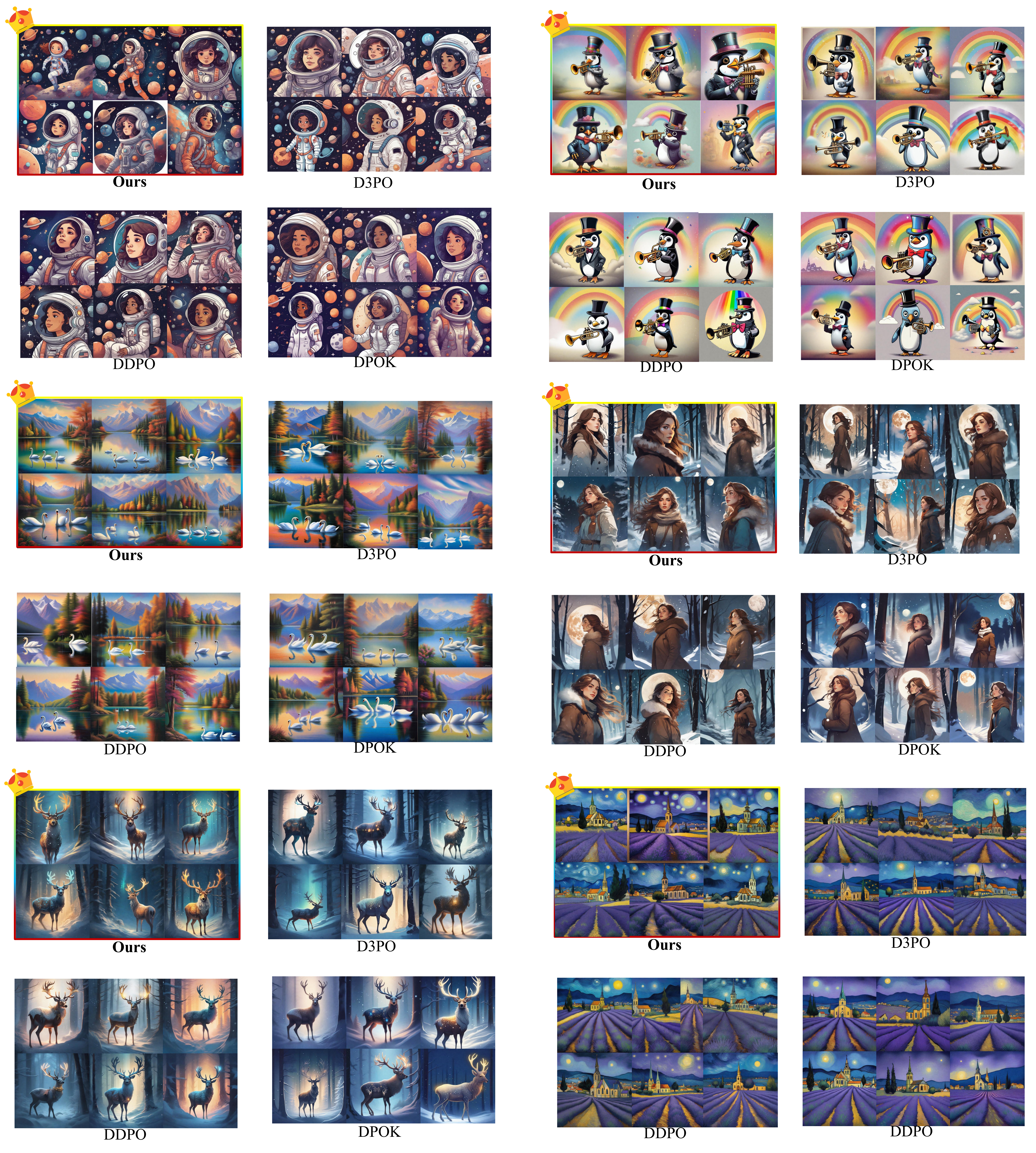}
    \caption{Diversity Experiments. Results demonstrate that \ourmethod{}  achieves superior diversity, notably in terms of object orientation, composition structure, and color richness.}
    \label{fig:div_all_past}
\end{figure*}

\begin{figure*}
    \centering
    \includegraphics[width=0.92\linewidth]{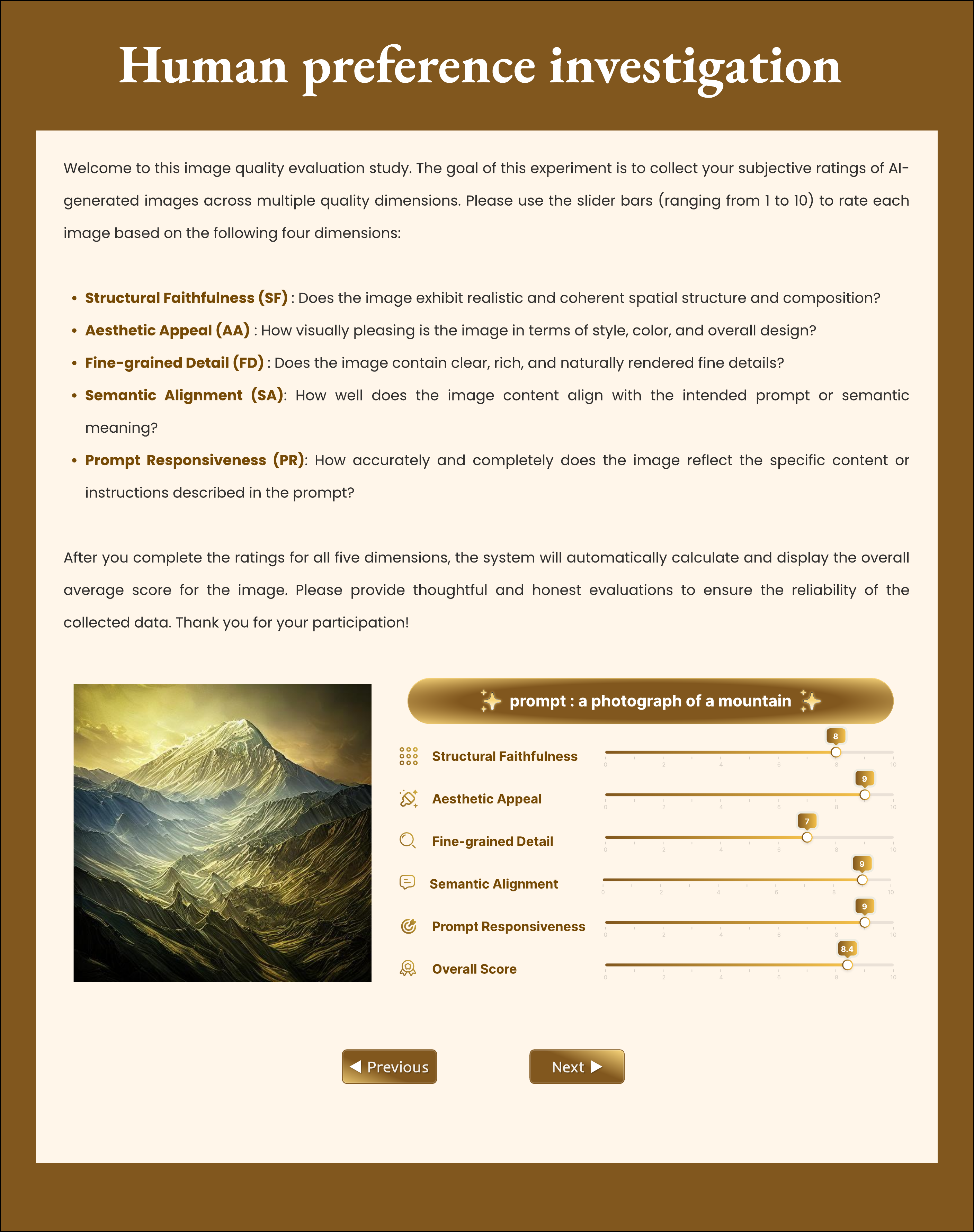}
    \caption{Interface for Subjective Evaluation}
    \label{fig:Interface}
\end{figure*}

\begin{figure*}
    \centering
    \includegraphics[width=0.99\linewidth]{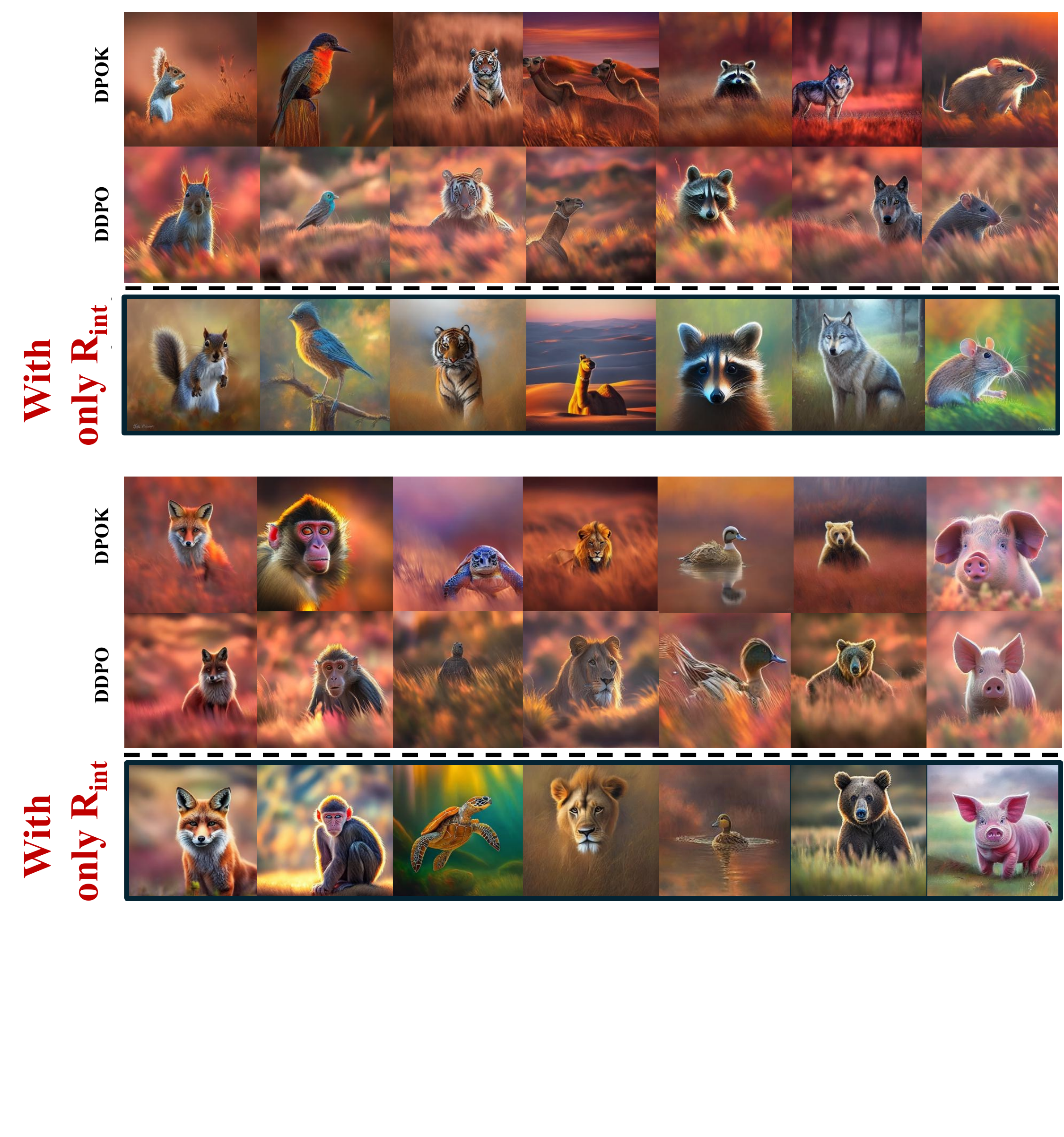}
    \caption{Intrinsic reward alleviates reward hacking: Using only intrinsic reward without other modules, we observe a significant reduction in reward hacking artifacts, notably the collapse into homogeneous styles and backgrounds.}
    \label{fig:hacking}
\end{figure*}

\begin{figure}[htbp] 
    \centering 
    \includegraphics[width=1\linewidth]{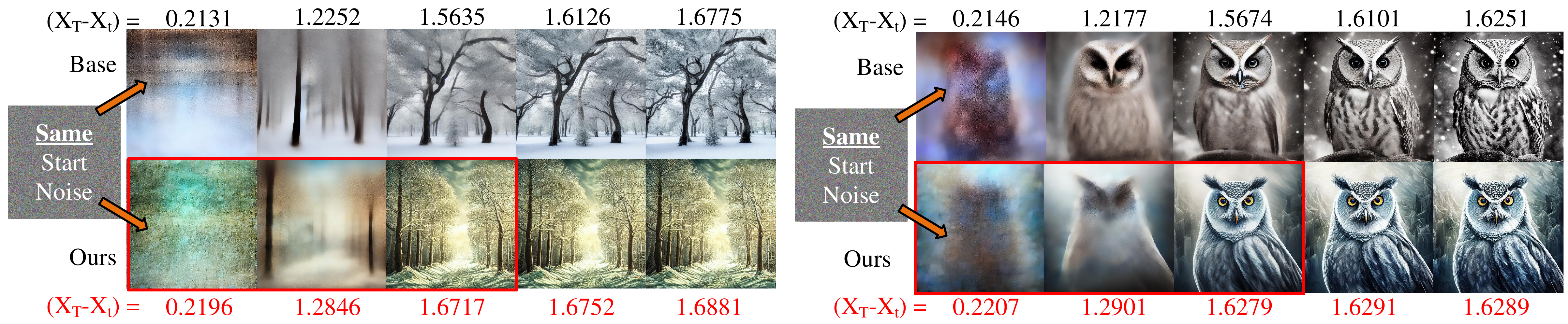} 
    \caption{Please zoom in. $R_{\text{int}}\left(\mathbf{s}_t, \mathbf{a}_t\right)$ effect on one denoising trajectory.} 
    \label{fig:tra_rebuttal} 
\end{figure} 

\noindent
\begin{minipage}{0.52\linewidth}
    \centering
    \includegraphics[width=\linewidth]{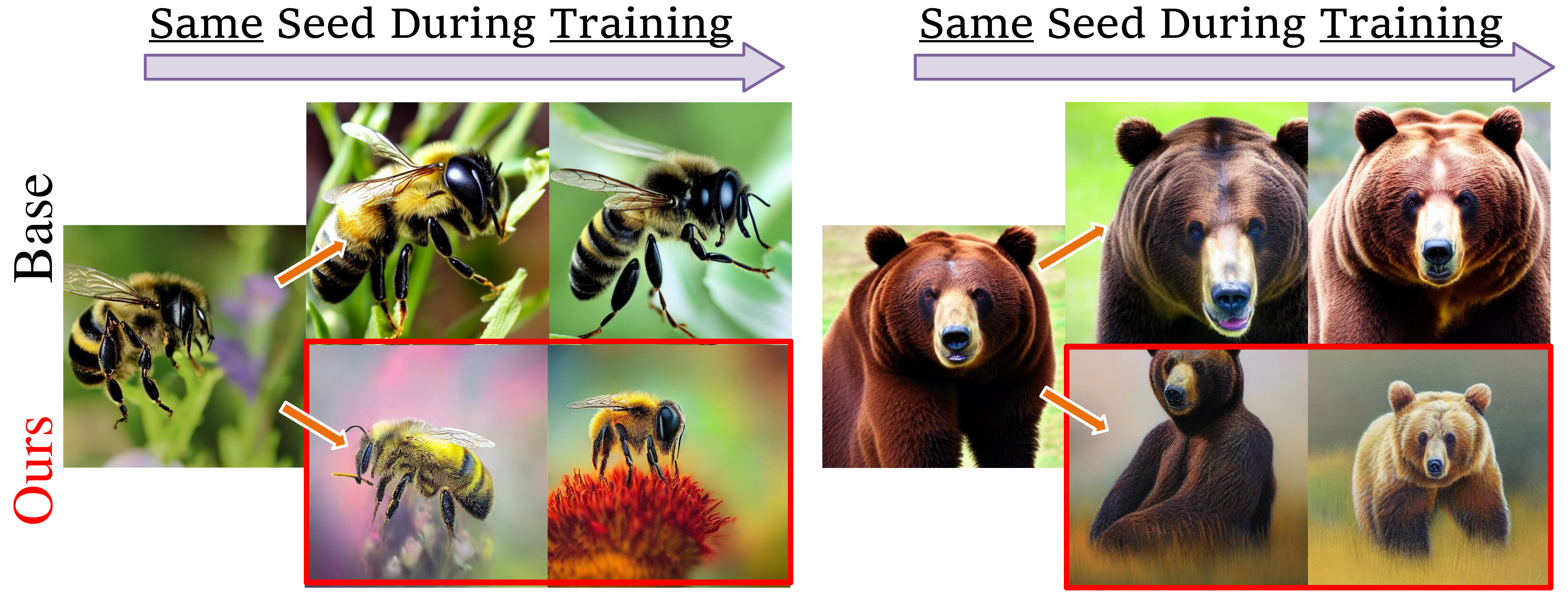}
    \captionsetup{hypcap=false}
    \captionof{figure}{$R_{int}$ effect on training.}
    \label{fig:epochD-re}
\end{minipage}
\hfill
\begin{minipage}{0.47\linewidth}
    \centering
    \captionsetup{hypcap=false}
    \captionof{table}{Fidelity \& Diversity.}
    \label{tab:tv}
    \begin{tabular}{lcccc}
    \toprule
    Method & FID$\downarrow$ & F1$\uparrow$ & LPIPS$\uparrow$ & TCE$\uparrow$ \\
    \midrule
    DDPO & 95.15 & 0.494 & 0.601 & 38.98 \\
    TDPO & 101.4 & 0.217 & 0.517 & 37.16 \\
    \rowcolor{blue!10}
    \textbf{Ours} & \textbf{50.19} & \textbf{0.857} & \textbf{0.651} & \textbf{40.07} \\
    \bottomrule
    \end{tabular}
\end{minipage}

\begin{table}[b]
\small
\centering
\caption{Detailed prompts used for generated images in Fig.~\ref{fig:div_all_past}.}
\label{tab:prompt_table}
\begin{tabular}{p{0.22\linewidth} p{0.70\linewidth}}
\toprule
\textbf{Image} & \textbf{Prompt} \\
\midrule
Row 1, Col 1 &
A girl astronaut exploring the cosmos, floating among planets and stars. \\
\midrule
Row 1, Col 2 &
A robot penguin wearing a top hat and playing a vintage trumpet under a rainbow. \\
\midrule
Row 2, Col 1 &
An artist captures the serene beauty of a tranquil lake surrounded by majestic mountains. Three swans glide gracefully across the water, mirroring the peaceful scene on his canvas. The vibrant colors of nature and the artist's focused dedication create a harmonious blend of art and reality in this picturesque setting.\\
\midrule
Row 2, Col 2 &
Under a bright moon, a brown-haired woman in a snowy forest; Splash art.\\
\midrule
Row 3, Col 1 &
A mystical deer with antlers that glow, guiding lost travellers through a snowy forest at night.
\\
\midrule
Row 3, Col 2 &
A night scene of a lavender field with a town and church in the background, reminiscent of Vincent van Gogh’s style.\\
\bottomrule
\end{tabular}
\end{table}

\subsection{Showcase Prompt Table}
Considering that the specific semantics of the prompts can substantially affect the assessment of the actual quality of generated images, it is necessary to assess the performance superiority of our method based on both images and their corresponding prompts.
Therefore, in Tab.~\ref{tab:prompt_table}, we provide a detailed list of the prompts that were not explicitly described in the main text.

\end{CJK*}

\end{document}